\documentclass[10pt]{article} 
\usepackage[accepted]{tmlr}

\usepackage{booktabs}
\usepackage{graphicx}
\usepackage{algorithm}
\usepackage{algpseudocode}
\usepackage{amsmath,amssymb,amsthm}
\usepackage{microtype}
\usepackage{placeins}
\usepackage{url}
\usepackage[hidelinks]{hyperref}

\newtheorem{theorem}{Theorem}
\newtheorem{lemma}{Lemma}
\newtheorem{proposition}{Proposition}
\newtheorem{corollary}{Corollary}
\theoremstyle{definition}
\newtheorem{definition}{Definition}
\newtheorem{remark}{Remark}

\newcommand{\R}{\mathrm{R}}
\newcommand{\Int}{\mathrm{Int}}
\newcommand{\tr}{\mathrm{tr}}
\newcommand{\one}{\mathbf 1}
\newcommand{\E}{\mathbb E}

\title{When Does Data Value Reduce to Class Balance?\\
A Coverage View of Per-Point Data Valuation}

\author{\name Yingjie Bai \email ybai0439@uni.sydney.edu.au \\
      \addr School of Computer Science, The University of Sydney
      \AND
      \name Weiming Zhi \email weiming.zhi@sydney.edu.au \\
      \addr School of Computer Science, The University of Sydney}

\def\month{09}
\def\year{2026}
\def\openreview{\url{https://openreview.net/forum?id=AZExWUZRPd}}

\begin{document}
\maketitle

\begin{abstract}
Per-point data-valuation scores --- Data Shapley, influence, leverage --- are routinely
used to \emph{select} training data: keep the points with the top $k$ scores. We argue
that \textbf{top-$k$ by a fixed, mode-blind per-point score is the wrong selection
rule precisely in the regime where data selection matters}. \emph{The reason} is that a
point's value is not a fixed number but depends on what else has been selected --- a near-duplicate of an already-chosen point adds almost nothing, however
high its score --- and a per-point scalar cannot express this interaction.

\emph{The theory} quantifies this failure exactly: in an equal latent-mode model,
fixed top-$k$ selection incurs a closed-form excess risk over adaptive coverage
because it repeatedly spends budget on already-covered modes. The gap is largest when
the budget is tight, the learner representation is redundant, and missing a mode is
costly.

\emph{The practical rule} is simple: use balance when the labels available at
selection identify the learner's modes, and use explicit coverage when they do not.
On standard classification tasks with strong pretrained features, class-balancing the
same per-point scores recovers the missing coverage and matches or outperforms generic
set coverage across our datasets. Without labels, label-free coverage improves accuracy
by up to $+0.20$ while raw top-$k$ can fall below random; when labels are coarser
than the modes, coverage within coarse strata performs best. We turn this distinction
into a pre-specified diagnostic and validate it across both favorable and
stress-test settings.

\emph{The limits} are part of the claim: the advantage is representation-relative.
Coverage fails when the relevant modes are not expressed in the learner geometry, as
with spurious subgroups, or when data quality varies independently of coverage ---
exactly as the theory predicts.

Code: \url{https://github.com/Benjy-byj/Data-Value-Class-Balance}
\end{abstract}

\section{Introduction}

\begin{figure}[t]\centering
\includegraphics[width=0.96\textwidth]{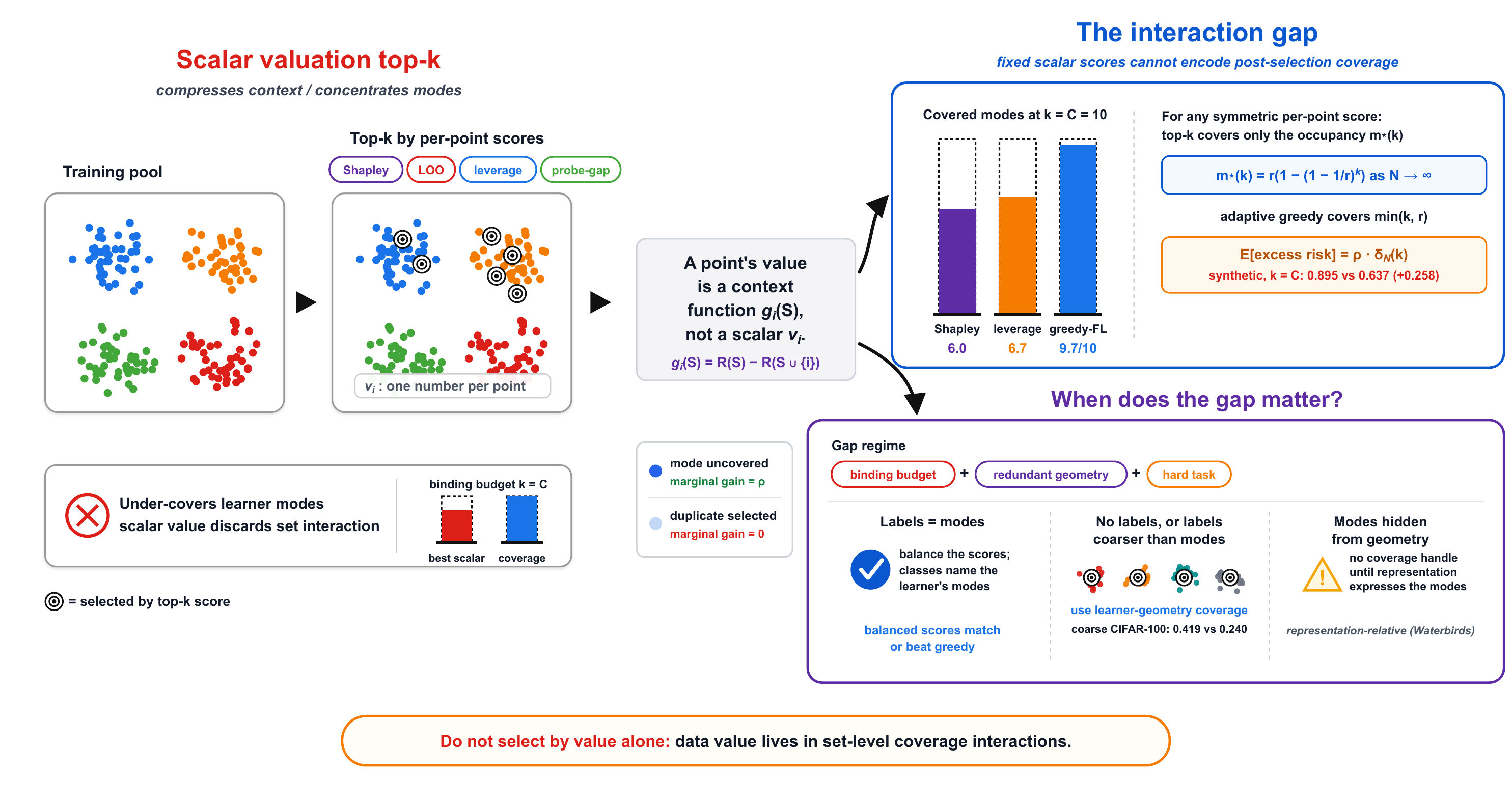}
\caption{\textbf{Data value is a context function, and a fixed scalar top-$k$ pays
an exact coverage price.} \emph{Left:} per-point valuation (Shapley, LOO/influence,
leverage, probe-gap) compresses each training point to one number $v_i$, and its
top-$k$ concentrates on a few modes of the pool (circled points), under-covering the
learner's modes: the scalar discards the set-level interaction. \emph{Middle:} a
point's value is the context function $g_i(S)=\R(S)-\R(S\cup\{i\})$ --- equal to
$\rho$ while the point's mode is uncovered, $0$ once a same-mode duplicate is
selected. \emph{Top right:} the exact result --- for any symmetric per-point score,
top-$k$ covers only the occupancy $m_\star(k)$ ($r(1-(1-1/r)^k)$ as $N\to\infty$)
versus adaptive greedy's $\min(k,r)$, giving
$\E[\text{excess risk}]=\rho\,\delta_N(k)$ (a lower bound for non-coordinated
scores, Definition~\ref{def:noncoord}); empirically, at $k=C=10$ Shapley covers $6.0$
and leverage $6.7$ of $10$ modes versus greedy's $9.7$, and the synthetic
binding-budget gap is $0.895$ vs $0.637$ ($+0.258$). \emph{Bottom right:} the gap
matters under a binding budget, a redundant learner geometry, and a hard task, and
the correction is the modes-vs-labels boundary: balance the same scores when the
labels name the learner's modes; use learner-geometry coverage when labels are
missing or coarser than the modes (coarse-label CIFAR-100: coverage within coarse strata $0.419$
vs.\ coarse-balanced prototypicality $0.240$ fine accuracy); and expect no coverage handle when the representation does not express the
modes (Waterbirds).}
\label{fig:teaser}
\end{figure}

Data valuation assigns each training point a scalar worth, and these scalars are
widely used to \emph{select} data: keep the top-$k$ most valuable points to train
cheaply, or to curate a fine-tuning set. Data Shapley~\citep{ghorbani2019},
leave-one-out and influence~\citep{koh2017}, datamodels~\citep{ilyas2022}, and
leverage scores all produce such per-point numbers. We ask whether top-$k$-by-score
is a good \emph{selection} rule. The answer is no, precisely in the regime where data
selection helps: under a binding budget, a redundant representation, and a
non-interpolable task, it is provably suboptimal in a latent-mode model. Worse,
``improving'' the per-point score can hurt selection when the sharper score
concentrates the selected set (Figure~\ref{fig:teaser} previews the argument).

The reason is structural. The value of a point is not a number but a function of the
\emph{context} of other selected points: a point that is the only representative of a
concept is valuable; once a near-duplicate is selected, its marginal value collapses.
A scalar summary discards exactly the information a selector needs: which points are
\emph{jointly} non-redundant. (Shapley averages the context function over contexts;
influence evaluates it at full data.) We make this precise, prove a lower bound on the
resulting selection gap, identify its source as the off-diagonal mass of the
learner-geometry Gram, and verify it.

\paragraph{Contributions.}
\begin{itemize}\itemsep2pt
\item \textbf{Value is a context function (\S\ref{sec:value}).} A closed form
$g_i(S)$ for kernel learners; the interaction is the off-diagonal (redundant) Gram
mass, and a fixed, mode-blind top-$k$ under-covers the modes relative to
greedy coverage.
\item \textbf{An impossibility (\S\ref{sec:theory}).} In an \emph{equal} latent-mode
instance, a symmetric valuation selected top-$k$ has expected excess risk \emph{exactly}
$\rho\,\delta_N(k)$ over greedy coverage, and at least this much for any
\emph{non-coordinated} randomized score (Definition~\ref{def:noncoord}; marginal
within-mode symmetry alone is provably insufficient, Remark~\ref{rem:coord}). The
positivity range and boundary cases are characterized exactly. LOO $\equiv0$,
Shapley is mode-constant, and a spiked-covariance analysis gives the soft-mode
extension.
\item \textbf{The modes-vs-labels boundary (\S\ref{sec:exp}).} Across synthetic and
real ViT experiments we show that the relevant question is not whether a scalar score is
accurate pointwise, but whether the groups available at selection name the learner's
modes. When modes coincide with labels, a one-line class balance of the same scores
recovers the missing coverage and can outperform label-free coverage. When labels are
unavailable or coarser than the modes, scalar balancing is insufficient and an explicit
learner-geometry coverage or diversity correction is needed. We also delimit the claim:
coverage helps only when the target modes are expressed in the learner geometry, and it
can fail when the representation hides the relevant subgroups.
\end{itemize}

\noindent The practical output of the paper is a one-line rule: \emph{if the labels
available at selection name the learner's modes, balance within labels; otherwise add an
explicit learner-geometry coverage correction}. When the modes are not expressed in
the representation at all, expect neither to help. \S\ref{sec:diagnostic} turns this
rule into a decision procedure (Algorithm~\ref{alg:rule}) and validates it, including
its failure modes.

\paragraph{Roadmap.} \S\ref{sec:value} sets up the one object everything else uses:
a point's value as a \emph{context function}, with a closed form for kernel learners
and the exact conditions under which its coverage objective is the supervised risk.
\S\ref{sec:theory} proves the main result in a hard latent-mode model
(\S\ref{sec:hardmodel}), states the extensions to general randomized scores and soft
modes (\S\ref{sec:extensions}), and collects scope and qualifications
(\S\ref{sec:scope}). \S\ref{sec:exp} organizes the experiments by role: direct
validation of the theory (\S\ref{sec:directval}), the modes-vs-labels boundary on real
data (\S\ref{sec:boundary}), the failure modes the theory predicts
(\S\ref{sec:failures}), the diagnostic (\S\ref{sec:diagnostic}), and robustness
(\S\ref{sec:broader}).

\section{Value is a context function}\label{sec:value}

\noindent\emph{What this section establishes.} Everything in the paper rests on one
reframing: a training point's value is not a scalar but a function of the
already-selected set. This section defines that context function and gives it a closed
form for kernel learners (Eq.~\ref{eq:gi}). It then states precisely when the coverage
objective it induces \emph{is} the supervised risk, and when it is only a proxy
(Proposition~\ref{prop:supervised}). Finally, it identifies the two mechanisms by
which redundancy makes value context-dependent
(Theorem~\ref{thm:nointeraction-restate}, Proposition~\ref{prop:intbound}) --- the
information a fixed scalar must discard. A closing paragraph fixes the terminology
(modes, classes, clusters) used throughout.

A learner has risk $\R(S)=\E_{x\sim D}(f_S(x)-f^\star(x))^2$ for a selected set
$S\subseteq P$, $|P|=N$. The \textbf{marginal value} of point $i$ in context $S$ is
$g_i(S)=\R(S)-\R(S\cup\{i\})\ge0$ (monotone $\R$). Scalar valuations collapse this
function to a number: leave-one-out / influence $v_i^{\mathrm{LOO}}=g_i(P\setminus\{i\})$;
Data Shapley $v_i^{\mathrm{Sh}}=\E_\pi g_i(S^\pi_{<i})$; leverage and probe-gap are
geometry- or label-derived per-point scores. The \textbf{interaction} is
$\Int_i=\max_S g_i(S)-\min_S g_i(S)$.

\paragraph{Linear-algebra form.} For a kernel/linear learner and \emph{point}
selection, the covered target is the feature second moment
$\Sigma=\E_x[\varphi\varphi^\top]$:
$\R(S)=\tr\big((I-P_S)\Sigma\big)$, with $P_S$ the projector onto
$\mathrm{span}\{\varphi_i:i\in S\}$. The marginal is the $\Sigma$-novelty of the
point's feature after orthogonalizing against $S$ (Gram--Schmidt / partial correlation):
\begin{equation}
g_i(S)=\frac{\langle\varphi_i^{\perp(S)},\,\Sigma\,\varphi_i^{\perp(S)}\rangle}{\|\varphi_i^{\perp(S)}\|^2},
\qquad \varphi_i^{\perp(S)}=(I-P_S)\varphi_i.
\label{eq:gi}
\end{equation}

\paragraph{What $\R(S)$ is, and when it is the supervised risk.} As written,
$\R(S)=\tr((I-P_S)\Sigma)$ is a \emph{representation-coverage} objective: it measures the
feature second-moment mass not reconstructed by the span of the selected points. It is not
automatically the supervised risk for an arbitrary target. The connection requires a
\emph{mode-aligned target} assumption: the target variation is supported on learner-mode
directions (equivalently, in the latent-mode instance the label is a function of mode
identity and an uncovered mode incurs a positive cost). The next proposition makes the
connection exact rather than heuristic.

\begin{proposition}[mode-aligned targets: coverage risk is the supervised excess
risk]\label{prop:supervised}
In the spiked-covariance model of \S\ref{sec:theory}
($\varphi_i=u_{j(i)}+\eta\,\xi_i$, $\Sigma=\sum_j\pi_ju_ju_j^\top+\eta^2P_\perp$), let the
target be mode-aligned, $f^\star(x)=\theta_\star^\top\varphi(x)$ with
$\theta_\star=\sum_j\beta_ju_j$, and let $\hat f_S$ be the best predictor whose
coefficient vector lies in $\mathrm{span}\{\varphi_i:i\in S\}$ (any kernel/ridge learner
trained on $S$ is of this form). Then
\[
\E\big(\hat f_S-f^\star\big)^2\;=\;\sum_{j:\,S\cap P_j=\emptyset}\beta_j^2\,\pi_j\;+\;O(\eta^4),
\]
i.e.\ the supervised excess risk \emph{equals} the coverage risk $\R(S)$ of the
latent-mode model with per-mode costs $\rho_j=\beta_j^2\pi_j$, exactly at $\eta=0$ and to
leading order for small $\eta$. For classification with labels constant on modes and an
argmax readout, the same coverage structure holds with $\rho_j$ the mode's probability
mass times its chance-level error: an uncovered mode is predicted at chance.
(Proof: App.~\ref{app:backbone}.)
\end{proposition}

\noindent Under mode-aligned targets, then, optimizing coverage \emph{is} optimizing
supervised risk; without the assumption, $\R(S)$ should be read as a coverage
\emph{proxy}. App.~\ref{app:modealign} tests the assumption directly on real features:
between $46\%$ and $74\%$ of the learner's target variance is constant within label-free
geometry modes, so the squared-loss equality is a hard-mode limit that real geometry
attenuates, while the classification clause holds --- the trained learner is correct on
only $0$--$5\%$ of the test points of uncovered classes. Three objects should not be conflated: \emph{(i)} the generic
supervised risk of the opening definition, about which we prove nothing in general;
\emph{(ii)} the representation-coverage objective $\tr((I-P_S)\Sigma)$, which the
closed form \eqref{eq:gi} and the greedy analysis concern; and \emph{(iii)} the hard
latent-mode risk of \S\ref{sec:theory}, where (ii) reduces to counting covered modes.
The exact theory establishes (iii) and extends to soft modes asymptotically. It
connects to (i) \emph{only} through the mode-aligned-target assumption of
Proposition~\ref{prop:supervised}: we do not claim that arbitrary supervised selection
utility reduces to geometric coverage, and Waterbirds is the standing counterexample. The experiments of \S\ref{sec:exp} then test
the proxy's empirical reach: coverage of learner-geometry modes tracks downstream
accuracy across the selectors, budgets, and datasets of the main experiments
(\S\ref{sec:broader}), including for an SGD-trained classifier.

\begin{theorem}[no-interaction geometry and redundancy]\label{thm:nointeraction-restate}
If the feature vectors are mutually orthogonal ($\varphi_i\perp\varphi_j$ for $i\ne j$),
then $g_i(S)$ is independent of any context $S$ not already containing $i$: selecting other
points never changes the residual of $\varphi_i$, so there is no selection interaction.
Conversely, when the learner geometry contains redundant directions, orthogonalizing
against previously selected neighbors can change a point's marginal in two distinct ways.
\emph{(i) Rank collapse:} the residual may fall to zero once a redundant neighbor is
selected, so $g_i$ collapses to $0$ --- an all-or-nothing interaction. \emph{(ii)
Directional rotation:} among contexts where the residual stays nonzero, its direction may
rotate, and if $\Sigma$ is anisotropic on those directions the Rayleigh quotient in
\eqref{eq:gi} changes. The hard latent-mode model below realizes mechanism~(i) exactly,
while the soft-mode and general-spectrum settings combine both. Both are carried by the
off-diagonal / block mass of the learner-geometry Gram $G=\Phi\Phi^\top$.
\end{theorem}

\begin{proposition}[interaction has two sources]\label{prop:intbound}
Let $\varphi_i^{\perp(S)}$ be the residual of $\varphi_i$ after projecting away
$\mathrm{span}\{\varphi_j:j\in S\}$. Interaction decomposes into two terms.
\emph{(i) Directional.} Over contexts $S,S'$ with $\varphi_i^{\perp(S)}\ne0$ and
$\varphi_i^{\perp(S')}\ne0$, write $w_i(S)=\varphi_i^{\perp(S)}/\|\varphi_i^{\perp(S)}\|$
and let $D_i=\sup\|w_i(S)-w_i(S')\|$ over such contexts be the diameter of the set of
residual directions. Then
\[
\sup_{S,S':\,\varphi_i^{\perp}\ne0}|g_i(S)-g_i(S')|\ \le\ \big(\lambda_{\max}(\Sigma)-\lambda_{\min}(\Sigma)\big)\,D_i.
\]
$D_i$ measures the amount of nonzero residual-direction variation induced by redundancy;
directional interaction requires $D_i>0$ \emph{and} anisotropy of $\Sigma$
($\lambda_{\max}>\lambda_{\min}$). Exact-duplicate or same-mode redundancy instead appears
through the separate rank-collapse term below.
\emph{(ii) Rank collapse.} If some context $S_0$ has $\varphi_i^{\perp(S_0)}=0$ (point $i$
lies in the selected span), then $g_i(S_0)=0$; this all-or-nothing term can be nonzero even
when $\Sigma$ is isotropic, and it is exactly the latent-mode mechanism (once a same-mode
point is selected, another point's residual collapses).
\end{proposition}
\begin{proof}
(i) $g_i(S)=R(w_i(S))$ with $R(w)=w^\top\Sigma w$. For any unit $w,w'$, in the eigenbasis
of $\Sigma$, $R(w)-R(w')=\sum_k(\lambda_k-c)(w_k^2-w_k'^2)$ for any $c$; with
$c=\tfrac12(\lambda_{\max}+\lambda_{\min})$ and Cauchy--Schwarz,
$|R(w)-R(w')|\le(\lambda_{\max}-\lambda_{\min})\|w-w'\|\le(\lambda_{\max}-\lambda_{\min})D_i$.
(ii) Since $\varphi_i$ lies in $\mathrm{span}(S_0)$, adding $i$ does not change the selected
span, so $\R(S_0\cup\{i\})=\R(S_0)$ and hence $g_i(S_0)=\R(S_0)-\R(S_0\cup\{i\})=0$.
\end{proof}

\noindent This separates two mechanisms a single off-diagonal summary would conflate.
\emph{Directional} interaction is the classical marginal-screening-versus-forward-selection
effect under collinearity: orthogonalization rotates the residual and anisotropy of
$\Sigma$ moves its Rayleigh quotient. \emph{Rank-collapse} interaction is stronger: a
duplicate or same-mode neighbor makes the residual vanish, collapsing the marginal to
zero. The latent-mode theorem below uses the all-or-nothing mechanism; the soft-mode
extension makes the collapse approximate.

Equivalently: scalar valuation is \emph{marginal (univariate) screening}; the optimal
context-aware selector is \emph{forward selection / greedy coverage} (which
orthogonalizes); their gap is the redundancy that orthogonalization removes. This is
the classical failure of marginal screening under collinearity, here applied to data
valuation. The off-diagonal Gram is a \emph{contribution covariance}: the selection
interaction lives in shared directions, not in per-point diagonals alone.

\paragraph{Terminology (modes vs.\ classes vs.\ clusters).} A \emph{mode} is a group of
points that the learner must cover to generalize --- a block of interchangeable features
in the learner geometry (a block of the Gram). \emph{Classes} are the labels available at
selection; \emph{clusters} are geometric groups in the representation. Our results turn on
whether these coincide: when the modes \emph{are} the classes, balancing labels suffices;
when they are finer geometric clusters (or unlabeled), coverage is needed. Throughout, $r$
is the number of modes, $m$ the per-mode size, $k$ the budget, and $C$ the number of
labels.

\section{The interaction gap}\label{sec:theory}

\noindent\emph{What this section establishes.} The main result
(\S\ref{sec:hardmodel}): in a hard latent-mode model, any \emph{fixed, symmetric}
per-point ranking used as top-$k$ pays an exact, computable excess-risk price
$\rho\,\delta_N(k)$ relative to adaptive greedy coverage.
Corollary~\ref{cor:phase} characterizes exactly when this price is large. Extensions to general
randomized scores and to soft modes are isolated in \S\ref{sec:extensions}; what the
result does \emph{not} say --- unequal modes, partially-aligned selectors, and the
precise procedure being critiqued --- is collected in \S\ref{sec:scope}. Readers who
want the argument's spine can read \S\ref{sec:hardmodel} alone.

\subsection{The hard latent-mode model and the exact gap}\label{sec:hardmodel}

\paragraph{Latent-mode instance.} Let the pool be $P=\bigsqcup_{j=1}^r P_j$ with $r$
modes of size $m=N/r$. A non-degenerate (block) learner geometry makes covering
$\ge1$ point of mode $j$ drive its residual to $0$; an uncovered mode costs $\rho_j>0$
(the task is non-interpolable across modes). Then $\R(S)=\sum_{j:\,S\cap P_j=\emptyset}\rho_j$.

\paragraph{Assumptions, stated explicitly, and their realism.} The hard instance makes
four idealizations. We state each one, together with where it is relaxed.
\emph{(A1) Rank-collapse coverage:} one representative drives its whole mode's residual
to zero. This is the $\eta\to0$ limit of the soft-mode model
(\S\ref{sec:extensions}, App.~\ref{app:soft}), where a representative captures the mode
direction only up to $O(\eta^2)$. The noise sweep (Table~\ref{tab:C}) shows the
empirical gap degrading continuously as the collapse becomes approximate. The same
dial, turned \emph{on real ViT features} (shrinking within-class variance toward the
block limit, or adding noise toward the diffuse one), raises the real-data gap
with block contrast through the datasets' operating point (Table~\ref{tab:blockness};
monotonically on Flowers-102, with one small local dip on CIFAR-100): the
hard model is the endpoint of a continuum the real data sits on, not a disconnected
idealization.
\emph{(A2) Block exchangeability:} within-mode points are indistinguishable to the
learner geometry and the valuation functional. Exact in the population model; on finite
noisy features it holds only in distribution, which is why Theorem~\ref{thm:gap} uses
uniform tie-breaking and Proposition~\ref{prop:general} requires non-coordination
within modes (Definition~\ref{def:noncoord}).
\emph{(A3) Equal sizes and costs:} used for the exact equality and for the unconditional
lower bound; Remark~\ref{rem:unequal} shows the bound can \emph{reverse} for unequal
sizes, so we do not claim it beyond the equal case.
\emph{(A4) Mode-aligned target:} the label varies across modes, not within
(Proposition~\ref{prop:supervised}); this is what licenses reading $\R(S)$ as supervised
risk, and Waterbirds (\S\ref{sec:failures}) is the deliberate failure case.
None of A1--A4 is claimed to hold literally on real data. The model isolates a
\emph{mechanism} --- fixed mode-blind rankings under-cover redundant geometry --- and the
experiments then test that mechanism where the assumptions hold only approximately
(pretrained ViT Grams have within-class block structure; off-diagonal mass
$0.37$--$0.39$, Table~\ref{tab:diag}) and where they fail (Waterbirds, the corrupted-pool
quality experiment).

\begin{lemma}[all-or-nothing marginal]\label{lem:bimodal}
For $i\in P_j$, $g_i(S)=\rho_j\,\one[S\cap P_j=\emptyset]\in\{0,\rho_j\}$, so
$\Int_i=\rho_j$. Any scalar $v_i\in[0,\rho_j]$ disagrees with the true marginal on an
$\Omega(1)$ fraction of contexts.
\end{lemma}
\begin{lemma}[exchangeability collapses scores]\label{lem:exch}
In the \emph{block-exchangeable} instance --- where points within a mode are
indistinguishable to the learner geometry and the valuation functional (the population /
hard-block model, or finite samples with identical within-mode features) --- any
symmetric valuation is mode-constant: $v_i=c_j$ for $i\in P_j$. (On finite noisy data
the within-mode scores are only equal in distribution; the occupancy bound then
requires more than marginal symmetry --- the precise sufficient condition is
non-coordination within modes, Definition~\ref{def:noncoord}, and
Remark~\ref{rem:coord} shows some such condition is necessary.)
\end{lemma}

\paragraph{The two canonical scalars, exactly.}
\emph{LOO/influence:} at full data every mode is covered, and removing one point of a
mode with $m\ge2$ leaves it covered, so $v_i^{\mathrm{LOO}}=0$ for \emph{every} point
--- influence is identically zero in the redundant regime. \emph{Shapley:}
$v_i^{\mathrm{Sh}}=\rho_j\Pr[i\text{ first of its mode in }\pi]=\rho_j/m_j$, constant
over points for equal modes. Both are mode-blind in the sense of Lemma~\ref{lem:exch}. (Numerically
confirmed: LOO $\equiv 0$; Shapley $=\rho/m$ to three digits.)

\begin{theorem}[interaction gap, exact for symmetric valuations]\label{thm:gap}
Take \emph{equal} modes ($m_j\equiv m$, $N=rm$), equal costs $\rho$, and any
budget $1\le k\le N$, with ties in the score broken uniformly at random. Because the modes
are statistically identical, a data-dependent \emph{symmetric} valuation (the population /
block-exchangeable versions of leave-one-out, Data Shapley, and leverage- or
probe-gap-type scores) is constant not only within each mode (Lemma~\ref{lem:exch}) but
\emph{across} modes, hence constant on the whole pool; its top-$k$ set is then a uniformly
random $k$-subset of $P$. (On finite noisy data these scores are constant only in
distribution; Proposition~\ref{prop:general} covers exactly the \emph{non-coordinated}
case, Definition~\ref{def:noncoord}, and Remark~\ref{rem:coord} delimits why some such
condition is needed.) It therefore covers exactly
the occupancy number of modes,
\[
\E[\#\text{covered}]=m_\star(k)
   =r\Big(1-\tbinom{(r-1)m}{k}\big/\tbinom{rm}{k}\Big)
   \xrightarrow[N\to\infty]{}\ r\big(1-(1-\tfrac1r)^k\big),
\]
while greedy coverage $S_{\mathrm{FL}}$ covers $\min(k,r)$ modes --- exactly optimal in
this instance, a fortiori $(1-1/e)$-optimal~\citep{nemhauser1978}. Hence, with
$\R(S)=\rho\,(r-\#\text{covered})$,
\[
\E[\R(S_v)]-\R(S_{\mathrm{FL}})=\rho\,\delta_N(k),\qquad
\delta_N(k)=\min(k,r)-m_\star(k)\ \ge 0,
\]
with strict positivity exactly on $2\le k\le N-m$ when $m\ge2$: at $k=1$ a single
uniform draw always covers a fresh mode, so $\delta_N(1)=0$ and top-$k$ is optimal; for
$k>N-m$ every $k$-subset covers all modes and the gap again vanishes; and if $m=1$ there
is no redundancy and $\delta_N\equiv0$ --- redundancy ($m\ge2$) is what creates the gap.
As $N\to\infty$ with
$k,r$ fixed, $\delta_N(k)\to\delta(k)=\min(k,r)-r\big(1-(1-1/r)^k\big)$; we use the
finite-pool $\delta_N$ in statements and the closed-form limit $\delta$ in prose.
\end{theorem}
\begin{proof}
Equal modes are exchangeable, so a symmetric functional of the data assigns them equal
values; with within-mode constancy (Lemma~\ref{lem:exch}) the score is constant on $P$,
and uniform tie-breaking makes $S_v$ a uniform $k$-subset. A fixed mode is missed iff
none of its $m$ points is drawn, probability $\binom{(r-1)m}{k}/\binom{rm}{k}$; summing
the complement over $r$ modes gives $\E[\#\text{covered}]=m_\star(k)$. Greedy gains a
fresh mode each step while an uncovered mode remains (a covered mode has zero marginal,
Lemma~\ref{lem:bimodal}), so it covers $\min(k,r)$. Subtract. For the positivity range:
$m_\star(1)=1$ (one draw, one mode); for $2\le k\le\min(r,N-m)$ the first draw covers a
fresh mode and each later draw repeats an already-hit mode with positive probability
(possible since $m\ge2$), so $m_\star(k)<k$; for $r<k\le N-m$ a fixed mode is missed with
positive probability, so $m_\star(k)<r$; and for $k>N-m$ every mode is hit surely.
(Full details, App.~\ref{app:gap}.)
\end{proof}

Two consequences give the theorem its practical reading: making a score \emph{better}
can make selection \emph{worse}, and the price is large exactly in the regime where
data selection is worth doing.

\begin{corollary}[improving a valuation need not improve selection]
Whether a valuation helps or hurts coverage depends on how its high values are spread
across modes. A score that \emph{concentrates} --- high values on many points of a few
dense, typical modes --- covers fewer than the uniform baseline and can fall \emph{below}
random at a binding budget; empirically the \emph{estimated} Data Shapley and
prototypicality do exactly this (their high scores cluster on typical points) and both
dip to or below random near $k=C$ in \S\ref{sec:exp}. The equal-mode symmetric valuation
instead ties random exactly (Theorem~\ref{thm:gap}). Either way, a more accurate
per-point number does not monotonically improve selection, because coverage is a
set-level property the scalar does not control.
\end{corollary}

\begin{corollary}[phase boundary]\label{cor:phase}
Two statements, both proved in App.~\ref{app:gap}. \emph{(a) Exact positivity range.}
The gap $\rho\,\delta_N(k)$ is nonzero precisely when $\rho>0$ (hard task), $m\ge2$
(redundant geometry), and $2\le k\le N-m$; outside this range a scalar top-$k$ is optimal
in the model. \emph{(b) Magnitude: when the gap matters.} In the large-$N$ limit the gap
is $\delta$-shaped in the budget: for $k\le r$,
$\delta(k)\le\binom{k}{2}/r$, so the per-budget regret $\delta(k)/k$ is $O(k/r)$ and
vanishes when $k\ll r$; at the binding budget $k=\Theta(r)$ the gap is $\Theta(r)$ --- a
constant fraction of the total risk $\rho r$ (e.g.\ $\delta(r)=r/e+O(1)$); and for
$k\ge r$ it decays as $\delta(k)=r(1-1/r)^k\le re^{-k/r}$. Hence the gap is a
constant-fraction excess risk iff all three conditions hold: (i) a \emph{binding budget}
($k=\Theta(r)$), (ii) a \emph{hard task} ($\rho>0$), and (iii) a \emph{redundant learner
geometry} ($m\ge2$, equivalently off-diagonal Gram mass $>0$ / effective rank $\ll N$).
Condition (iii) is geometry redundancy, not representation non-locality.
\end{corollary}

\noindent All three claims are confirmed numerically
(App.~\ref{app:num}, Table~\ref{tab:theorycheck}): uniform coverage agrees with the
closed-form $m_\star(k)$ within $0.007$ modes
(Theorem~\ref{thm:gap}), a concentrating score stays below it
(Proposition~\ref{prop:general}), and the unequal-mode reversal of
Remark~\ref{rem:unequal} occurs as predicted.

\subsection{Extensions: general randomized scores and soft modes}\label{sec:extensions}

\noindent\emph{What this subsection adds.} The theorem covers symmetric scores with
uniform tie-breaking. Two extensions matter in practice: scores that are only
\emph{randomized} approximations of a mode-level quantity (the finite-noisy-data case),
and modes that are only \emph{approximately} redundant. Both are stated here, separated
from the main result; neither changes its message.

\begin{definition}[non-coordinated scores]\label{def:noncoord}
A (possibly randomized) score is \emph{non-coordinated within modes} if it decomposes as
$v_i=c_{j(i)}+\varepsilon_i$, where $(c_1,\dots,c_r)$ is any random mode-level vector
(it may depend arbitrarily on the data) and, conditionally on $(c_j)$ and the mode
partition, the within-mode perturbations $(\varepsilon_i)$ are i.i.d.\ across the points
of each mode and independent across modes (atomless, so ties have probability zero; a
deterministic mode-constant score with uniform tie-breaking is the special case of
i.i.d.\ uniform $\varepsilon$).
\end{definition}

\begin{proposition}[general lower bound, equal modes, non-coordinated scores]\label{prop:general}
For \emph{equal-size} modes ($m\ge2$), any score that is non-coordinated within modes
(Definition~\ref{def:noncoord}) used as top-$k$ covers at most the occupancy,
$\E[\#\text{covered}]\le m_\star(k)$, so
$\E[\R(S_v)]-\R(S_{\mathrm{FL}})\ge\rho\,\delta_N(k)$; for $2\le k\le N-m$ equality holds
iff the mode-level values are equal (Theorem~\ref{thm:gap}), while at $k=1$ every
selector covers exactly one mode and the bound is trivially tight. A score whose
mode-values \emph{concentrate} the budget on a few modes covers strictly fewer
(App.~\ref{app:gap}, by a pooling / Schur-convexity argument that uses the equal mode
sizes).
\end{proposition}

\begin{remark}[coordination is exactly what the bound excludes --- and exactly
set-level information]\label{rem:coord}
Some restriction beyond marginal within-mode symmetry is \emph{necessary}: a
randomized score that draws one representative per mode uniformly and assigns it value
$1$ (all others $0$) has within-mode exchangeable marginals --- no point is
systematically preferred --- yet its top-$r$ covers all $r$ modes, beating the
occupancy bound. But such a score is \emph{jointly coordinated}: its within-mode
components are anti-correlated in a way that requires knowing the mode partition ---
which is precisely set-level (mode-coordinating) information. Definition~\ref{def:noncoord}
excludes exactly this: a per-point score computed pointwise from the data, with
independent per-point noise, cannot coordinate across a mode. The impossibility is
therefore sharp about what ``per-point'' means: to beat occupancy, mode-coordinating
structure must enter the scores. That is escape (iii) of the scope discussion below (a
mode-aware score is no longer mode-blind), and it is exactly what balancing and
coverage corrections supply.
\end{remark}

\paragraph{Soft modes / general spectra (asymptotic in $\eta$).} Replace hard modes by
a spiked covariance $\Sigma=\sum_j\pi_j u_ju_j^\top+\eta^2P_\perp$ with within-mode
noise $\eta$, under the explicit assumptions (S1)--(S3) of App.~\ref{app:soft}
(orthonormal spikes; \emph{equal} masses $\pi_j\equiv\pi=:\rho$ for the equal-mode
statement; noise directions $\xi_i$ i.i.d., mean-zero, supported on the spikes'
orthogonal complement, and bounded). Then (App.~\ref{app:soft}) there exist
$\eta_0>0$ and $c<\infty$, depending only on the instance parameters, such that for all
$\eta\le\eta_0$ every uncovered-mode marginal is $\ge\rho-c\eta^2$ and every
covered-mode marginal is $\le c\eta^2$, whence
$\E[\R(S_v)]-\R(S_{\mathrm{FL}})\ \ge\ \rho\,\delta_N(k)-c\,\eta^2k$ for
$2\le k\le N-m$ --- positive for small $\eta$. We state this as an \emph{asymptotic}
extension and do not optimize the constants (numerically, at our settings, the two
marginals are $0.117\approx\pi$ and $0.004$); unequal masses assign unequal effective
costs $\rho_j=\pi_j$ and fall under the caveats of Remark~\ref{rem:unequal}.
The $\pi_j$ are the spike eigenvalues of $\Sigma$ and $\eta^2$ the tail: the gap is large
for a \emph{peaked} spectrum (low effective rank, redundant) and vanishes as the
spectrum \emph{flattens} (high effective rank, near-orthogonal). General $\Sigma$
reduces to the marginal-screening-vs-OMP gap of sparse approximation, governed by the
Gram coherence.

\subsection{Scope and qualifications}\label{sec:scope}

\noindent\emph{What this subsection delimits.} Three boundaries of the result, stated
so they are not discovered by surprise: selectors partially aligned with the modes
(open), unequal mode sizes (the bound can reverse), and the precise procedure the
theorems critique (fixed, mode-blind, non-adaptive rankings --- nothing else).

\noindent Informally, a selector whose geometry is only \emph{partially} aligned with the
modes should interpolate between greedy coverage (fully aligned) and the mode-blind
occupancy baseline; formalizing this interpolation for general representations is left
open.

\begin{remark}[unequal modes]\label{rem:unequal}
With \emph{unequal} sizes the bound can reverse: a score favoring \emph{small} (rare)
modes can cover \emph{more} than uniform. With $r=10$ modes --- nine singletons and one
mode of $91$ points --- and $k=9$, a score that prefers the singletons covers all nine,
beating uniform random. Coverage then depends on whether the score's high values align
with small modes (spreading, helpful) or large ones (concentrating, harmful); only
equal modes give the unconditional bound.
\end{remark}

\paragraph{Scope: what is and is not critiqued.} The result is about a specific
procedure: a \emph{fixed, mode-blind, non-adaptive} per-point ranking used as top-$k$,
with no balancing, grouping, or coverage constraint. Three escapes leave this regime, and the rest
of the paper is precisely about them: (i) \emph{adaptive} recomputation of the marginal
after each pick is exactly greedy and covers $\min(k,r)$; (ii) a \emph{group-aware} score
(class quotas, subgroup balance) restores coverage when the groups are the modes
(\S\ref{sec:boundary}); and (iii) a score that already encodes mode-aware structure is, by
definition, no longer mode-blind. What we claim is therefore sharp but bounded:
$\delta_N(k)$ is the exact \emph{price of non-adaptivity} --- the regret of any fixed
mode-blind score relative to adaptive coverage --- so \emph{within this regime} the
remedy is structural (balance or cover), not a more accurate number. We do not claim that every
conceivable scalar fails; we identify the regime in which a naive top-$k$ provably
under-covers.

\section{Experiments}\label{sec:exp}

\paragraph{Findings at a glance.}
\emph{Scope of the empirical claims.} All real-data experiments select from
\emph{frozen} pretrained representations with probe-family learners (ridge, an
SGD-trained MLP, logistic regression, $1$-NN), extended by a from-scratch CNN pilot
and a LoRA fine-tuning check (App.~\ref{app:robust}). The evaluation emphasizes the
binding budget $k\approx C$, where the theory locates the largest gap; full
$k=C,2C,4C$ sweeps are in App.~\ref{app:robust} (Table~\ref{tab:budgets}). The
strongest empirical claims should be read within this setting.

The section is organized by the \emph{role} each experiment plays in the argument:
direct tests of the formal results (\S\ref{sec:directval}), the practical
modes-vs-labels claim on real features (\S\ref{sec:boundary}), the failure modes the
theory predicts (\S\ref{sec:failures}), the decision procedure built on them
(\S\ref{sec:diagnostic}), and robustness to the estimator, encoder, modality, and
hyperparameters (\S\ref{sec:broader}). One line per finding:

\begin{itemize}\itemsep1pt
\item \emph{Direct validation} (\S\ref{sec:directval}): a symmetric valuation's
top-$k$ agrees with the closed-form $m_\star(k)$ within $0.007$ modes, and every boundary case of the corrected
theorem is verified (Table~\ref{tab:theorycheck}). On real pools the same formula
predicts random selection's coverage with zero parameters, and every raw scalar
top-$k$ falls below it (Table~\ref{tab:occupancy}). The theory's exact selector
provably and empirically matches facility location in the block regime and diverges
exactly where the lemma's condition fails (Lemma~\ref{lem:flproj},
Tables~\ref{tab:B},~\ref{tab:omp}).
\item \emph{The boundary} (\S\ref{sec:boundary}): with labels $=$ modes,
class-balancing the \emph{same} scores matches or beats greedy coverage
(Table~\ref{tab:real}),
and KNN-Shapley concentrates like every other scalar (Table~\ref{tab:knnshap}); with
no labels, coverage beats random by $+0.08$ to $+0.20$ while raw top-$k$ falls below
random (Table~\ref{tab:unlab}); with labels coarser than modes, the correction is
stratified coverage, not scalar balancing (Table~\ref{tab:cf}).
\item \emph{Failure modes} (\S\ref{sec:failures}): coverage fails when the
representation does not express the modes (Waterbirds) and when quality varies
independently of coverage (corrupted pools, Table~\ref{tab:quality}) --- both
predicted by the theory's conditions.
\item \emph{The diagnostic} (\S\ref{sec:diagnostic}): validated on every setting,
including an honest near-threshold misroute on coarse labels that the pre-stated
ambiguity rule catches, and a subgroup audit row that detects Waterbirds' hidden modes
(Table~\ref{tab:diag}).
\item \emph{Robustness} (\S\ref{sec:broader}, App.~\ref{app:robust}): the gap
persists across probes, an SGD-trained MLP, logistic/1-NN, three encoders, a text
corpus, a from-scratch end-to-end pilot, and a LoRA fine-tune of the backbone; coverage tracks accuracy across the $244$ pooled pairs of
\S\ref{sec:broader}.
\end{itemize}

\subsection{Direct validation of the theory}\label{sec:directval}

\noindent\emph{What this subsection shows.} The formal results, tested as literally
as possible: the synthetic suite instantiates the latent-mode regime and checks the
occupancy predictions, the concentration mechanism, the causal role of redundancy, and
the phase boundary. Two further direct checks on \emph{real} pools --- the occupancy
formula as a zero-parameter prediction (Table~\ref{tab:occupancy}) and the blockness
dial on real features (Table~\ref{tab:blockness}) --- appear in App.~\ref{app:robust}.

We instantiate the latent-mode regime synthetically ($C{=}10$ classes, $3$
input-distant modes each, a non-degenerate embedding; learner $=$ ridge probe). The
embedding \emph{collapses} each class's three input-space clusters onto a single
near-orthogonal class direction, so the \emph{learner geometry} has $r=C=10$ modes
even though the input space has $30$ clusters; modes are defined by the learner
geometry throughout (\S\ref{sec:value}), which is why the figures state $r=C=10$. The
full suite runs in seconds (numpy). Table~\ref{tab:B} is the core result: every
per-point scalar, selected top-$k$, is beaten by greedy coverage at the binding budget
$k\!=\!C$, and the gap is $\delta$-shaped in the budget. LOO is near-useless, and
Shapley stays far below the context-aware selector. Figure~\ref{fig:synthetic}
summarizes the gap, its mechanism, and its cause.

\begin{table}[t]\centering\small
\caption{Test accuracy by selector vs.\ budget ($r=C=10$). \textbf{Greedy-FL}
(context-aware coverage) beats every per-point top-$k$ at the binding budget; the gap
is $\delta$-shaped (peaks at $k=C$, $\to0$ as $k\gg C$). \textbf{Greedy-proj} is the
theory's \emph{exact} selector (greedy trace-reduction / OMP on the feature second
moment, Eq.~\ref{eq:gi}); in this near-block regime it matches facility location at
binding budgets ($0.889$ vs $0.895$ at $k=C$; both cover ${\approx}9.7$--$9.8$ of $10$
classes), as Lemma~\ref{lem:flproj} predicts; App.~\ref{app:details}
(Table~\ref{tab:omp}) delimits where the correspondence breaks on real features. Gaps are computed from unrounded means, so they can differ by $0.001$ from differences of the displayed values.}
\setlength{\tabcolsep}{4.5pt}
\begin{tabular}{rrrrrrrrrr}
\toprule
$k$ & random & LOO & leverage & proto & marginal & shapley & greedy-FL & greedy-proj & FL$-$best \\
\midrule
5  & 0.419 & 0.249 & 0.444 & 0.417 & 0.293 & 0.390 & \textbf{0.490} & 0.490 & $+0.047$ \\
10 & 0.601 & 0.263 & 0.600 & 0.637 & 0.417 & 0.615 & \textbf{0.895} & 0.889 & $+0.258$ \\
15 & 0.690 & 0.267 & 0.741 & 0.719 & 0.500 & 0.678 & \textbf{0.909} & 0.906 & $+0.168$ \\
20 & 0.821 & 0.364 & 0.802 & 0.847 & 0.526 & 0.757 & \textbf{0.928} & 0.886 & $+0.081$ \\
30 & 0.821 & 0.343 & 0.819 & 0.907 & 0.612 & 0.837 & 0.908 & 0.862 & $+0.001$ \\
50 & 0.897 & 0.464 & 0.894 & 0.951 & 0.729 & 0.924 & 0.942 & 0.870 & $-0.010$ \\
\bottomrule
\end{tabular}
\label{tab:B}
\end{table}

\paragraph{Value is not a scalar.} The marginal $g_i(S)$ over random contexts has
spread/$|$mean$|\approx1.5$ and is bimodal ($40\%$ large gain, $57\%\approx0$) ---
all-or-nothing, as Lemma~\ref{lem:bimodal} predicts.

\paragraph{The mechanism, made visible: scalars concentrate.} The failure is directly
visible in the number of \emph{classes covered} by the selected set. At the binding budget
$k=C=10$, greedy coverage selects points spanning $9.7$ of the $10$ classes, whereas
every per-point top-$k$ covers far fewer --- leverage $6.7$, prototypicality $6.5$,
\emph{Data Shapley $6.0$}. The reason is exactly the concentration in the proof: a
per-point score is
(near-)constant within a class, so the top-$k$ fills up the highest-scoring
class(es) before reaching others, leaving uncovered classes uncovered. The
\emph{more} a score concentrates on a small set of high-scoring classes, the fewer
classes it covers --- the counterintuitive corollary as a picture.

\paragraph{Redundancy causes the gap (causal).} Sweeping the embedding noise to tune
redundancy, the gap at $k=C$ falls with the off-diagonal / block structure
(Table~\ref{tab:C}): as the effective rank rises from $15.5$ to $29.8$ the gap collapses
from $+0.35$ to within $0.025$ of zero ($-0.024$ and $+0.003$ at the two most diffuse
settings). This is a controlled causal law, not a cross-task correlation.

\begin{table}[t]\centering\small
\caption{Redundancy drives the gap. As noise rises the block-contrast (off-diagonal
mass) falls and the effective rank rises; the gap $=$ greedy-FL $-$ leverage-top-$k$
at $k=C$ collapses.}
\begin{tabular}{rrrr}
\toprule
noise & block-contrast & eff.\ rank & gap @ $k{=}C$ \\
\midrule
0.10 & 0.416 & 15.5 & $+0.349$ \\
0.18 & 0.231 & 22.8 & $+0.295$ \\
0.30 & 0.105 & 27.7 & $+0.035$ \\
0.50 & 0.040 & 29.4 & $-0.024$ \\
0.80 & 0.014 & 29.8 & $+0.003$ \\
\bottomrule
\end{tabular}
\label{tab:C}
\end{table}

\begin{figure}[t]\centering
\includegraphics[width=0.85\textwidth]{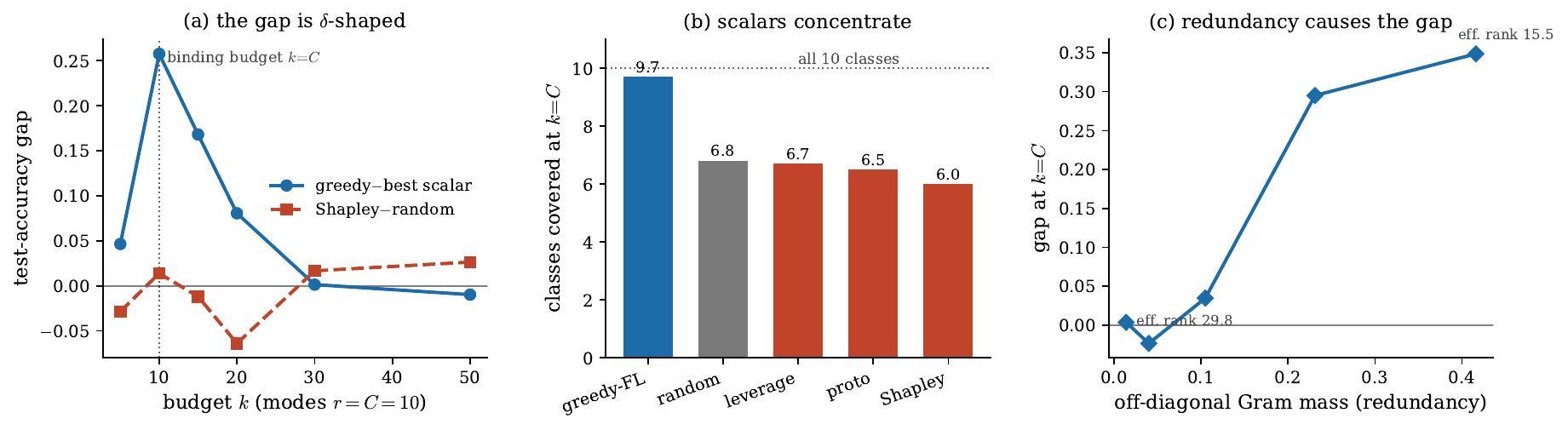}
\caption{\textbf{The interaction gap, its mechanism, and its cause (synthetic,
$r=C=10$).} (a) The gap greedy$-$best-scalar is $\delta$-shaped: it peaks at the
binding budget $k=C$ ($+0.258$) and vanishes as $k\gg C$; top-$k$-by-Shapley even falls
\emph{below random} near the binding budget (red). (b) Mechanism: at $k=C$ greedy
coverage spans $9.7/10$ classes while every per-point scalar concentrates and covers far
fewer (Shapley, the most mode-informative, covers fewest, $6.0$). (c) Cause: sweeping
embedding noise, the $k=C$ gap grows with the off-diagonal Gram mass
(redundancy) and collapses to ${\approx}0$ as the effective rank rises from $15.5$ to $29.8$ --- a
controlled causal law, not a cross-task correlation.}
\label{fig:synthetic}
\end{figure}

\paragraph{Phase boundary.} Crossing the three conditions of Corollary~\ref{cor:phase},
the gap appears only under \emph{binding budget $\wedge$ hard task $\wedge$ redundant
geometry}, where it is largest ($+0.278$). Diffuse or loose-budget cells collapse to
$\le0$; the redundant
easy/tight cell leaves only a smaller finite-probe residual. Representation locality
is not a condition (that distinguishes this from the selection-geometry advantage it
generalizes).

\subsection{The modes-vs-labels boundary (real ViT-B/16 features)}\label{sec:boundary}

\noindent\emph{What this subsection shows.} The theory says the fix for scalar
concentration is structural; \emph{which} structure depends on whether the labels
available at selection name the learner's modes. We instantiate all three cases on
frozen ViT-B/16 features (learner $=$ ridge probe; PCA-64; 3--5 seeds): labels $=$
modes (balance wins), no labels (coverage wins), labels coarser than modes
(stratified coverage wins).

\paragraph{Modes $=$ labels: balancing suffices, and matches or beats coverage.} A \emph{raw}
per-point top-$k$ fails badly at the binding budget $k=C$: it concentrates and
under-covers classes, losing $+0.12$ to $+0.22$ to greedy coverage
(Table~\ref{tab:real}, first row). But pretrained ViT collapses each class to
roughly one cluster, so here the \emph{modes are the labels} (the operating point:
off-diagonal Gram mass $0.37$--$0.39$ on these datasets, Table~\ref{tab:diag}; the
severity of the raw-top-$k$ pathology is encoder-dependent,
Table~\ref{tab:encoders}). Class-balancing the very same scores (equal quota per
class, ranked within) then \emph{recovers} the coverage and matches or beats greedy
facility location at every budget (Table~\ref{tab:budgets}). The headline comparison is \emph{deployable}, not oracle: at $k=C$ the single
pre-specified selector \emph{class-balanced prototypicality} beats greedy facility
location by $0.09$, $0.08$, and $0.17$ on CIFAR-10, CIFAR-100, and Flowers-102 and ties
it on Oxford Pets ($0.004\pm0.015$; Table~\ref{tab:real}), and it coincides with the
best balanced scalar throughout (Table~\ref{tab:fullsel}). So on these datasets ---
standard classification, strong pretrained features, labels that name the learner's
modes --- the selection interaction is class coverage, and a one-line class balance is
the best fix we measure. One caveat on scope: this is a claim about the
modes-$=$-labels regime, certified per-dataset by the alignment diagnostic, not a
claim that data value generally reduces to class balance. The no-label, coarse-label,
and Waterbirds settings below are exactly the cases where it does not.

\begin{table}[t]\centering\small
\caption{Modes equal labels, using ViT features at budget $k=C$, with standard errors over
five seeds. A raw per-point top-$k$ loses to greedy facility location, but class-balancing
the same scores matches or beats greedy facility location; in this regime the apparent interaction
gap is class imbalance, fixed by balancing. The comparison is \emph{deployable}, not
oracle: the balanced row is a \emph{single pre-specified selector}, class-balanced
prototypicality (which coincides with the best balanced scalar on every dataset,
Table~\ref{tab:fullsel}); the raw row keeps the oracle-\emph{best} raw scalar, which is
conservative for our claim (coverage beats even the best raw score). Per-selector matched
comparisons: Table~\ref{tab:fullsel}.}
\begin{tabular}{lcccc}
\toprule
& CIFAR-10 & Oxford Pets & CIFAR-100 & Flowers-102 \\
\midrule
greedy $-$ best \emph{raw} scalar (oracle) & $+0.196{\scriptstyle\,\pm0.042}$ & $+0.166{\scriptstyle\,\pm0.021}$ & $+0.119{\scriptstyle\,\pm0.012}$ & $+0.223{\scriptstyle\,\pm0.006}$ \\
greedy $-$ balanced prototypicality & $-0.092{\scriptstyle\,\pm0.020}$ & $-0.004{\scriptstyle\,\pm0.015}$ & $-0.081{\scriptstyle\,\pm0.012}$ & $-0.171{\scriptstyle\,\pm0.006}$ \\
\bottomrule
\end{tabular}
\label{tab:real}
\end{table}

\paragraph{A mainstream valuation method concentrates too.} To check the failure is not
specific to leverage or prototypicality, we add \emph{exact} KNN-Shapley~\citep{jia2019}
--- a real Data-Shapley method --- on the same frozen features
(Table~\ref{tab:knnshap}). Used as a raw top-$k$ it behaves identically: on CIFAR-10 it
covers $5.3$ of $10$ classes and scores $0.43$, \emph{below} random ($0.54$); on
CIFAR-100 it trails random. Class-balancing the very same Shapley values restores full
coverage and lifts accuracy ($0.68$, $0.36$). Per-point data valuation thus concentrates
exactly as predicted --- the critique is about the per-point ranking \emph{procedure}, not
one score. Appendix~\ref{app:robust} extends the test to the training-dynamics family
(EL2N, GraNd, TracIn-CP): every raw top-$k$ under-covers there too, balancing restores
coverage, and the difficulty orientation adds a quality axis that balancing alone cannot
fix --- the decoupled quality-then-cover recipe of \S\ref{sec:limits} on real scores.

\begin{table}[t]\centering\small
\caption{\textbf{A real Data-Shapley method (KNN-Shapley) concentrates like the other
scalars} (frozen ViT, $k=C$, $3$ seeds; accuracy / classes covered). Values are computed
on a held-out validation split \emph{drawn from the training pool}, disjoint from the final
test set. Raw top-$k$ trails random and covers few classes; balancing the same values
recovers; greedy coverage is strongest. The mechanism is the per-point ranking, not the
particular score.}
\setlength{\tabcolsep}{6pt}
\begin{tabular}{lcc}
\toprule
selector ($k=C$) & CIFAR-10 (acc\,/\,cov\,$/10$) & CIFAR-100 (acc\,/\,cov\,$/100$) \\
\midrule
random & $0.541$ / $6.3$ & $0.265$ / $64.0$ \\
KNN-Shapley, raw top-$k$ & $0.433$ / $5.3$ & $0.245$ / $60.3$ \\
leverage, raw top-$k$ & $0.352$ / $4.7$ & $0.190$ / $55.3$ \\
KNN-Shapley, class-balanced & $0.682$ / $10.0$ & $0.363$ / $100.0$ \\
greedy facility location & $\mathbf{0.779}$ / $9.0$ & $\mathbf{0.391}$ / $76.7$ \\
\bottomrule
\end{tabular}
\label{tab:knnshap}
\end{table}

\paragraph{No labels to balance: label-free coverage wins across datasets.} When
selection has \emph{no} labels (the common curation setting), class-balancing is
unavailable and the only label-free options are random, a per-point scalar top-$k$, or
set coverage. The modes are then the many fine classes --- geometric clusters in ViT
space --- exactly the regime the theory targets. (The budget is set to $k=C$ using the
known class count only to place the evaluation at the theoretically binding point; the
selectors never see labels, nothing in the methods requires knowing $C$, and
Table~\ref{tab:budgets} reports the same comparisons at $2C$ and $4C$.) Across three
many-class datasets at the binding budget $k=C$ (Table~\ref{tab:unlab}, $5$ seeds),
greedy facility location covers far more classes and wins on accuracy by $+0.08$ to
$+0.20$ over random (many standard errors). A \emph{raw per-point top-$k$} (leverage
or global typicality) instead concentrates: it covers \emph{fewer} classes than random
and falls \emph{below} random accuracy on all three datasets --- the counterintuitive
corollary, now on real data. The gap is $\delta$-shaped (largest at $k=C$, shrinking
by $4C$).

\begin{table}[t]\centering\small
\caption{No labels to balance (frozen ViT features, $k=C$, mean $\pm$ se over $5$ seeds).
With no labels, label-free greedy coverage beats random on every dataset by many standard
errors, while a \emph{raw} per-point top-$k$ (leverage / global typicality) concentrates
and falls \emph{below} random on both accuracy and classes covered on all three datasets
--- the concentration mechanism on real data.}
\setlength{\tabcolsep}{3pt}
\begin{tabular}{l ccccc}
\toprule
$k=C$ & random & leverage & typicality & $k$-center & greedy-FL \\
\midrule
\multicolumn{6}{l}{\emph{test accuracy}}\\
CIFAR-100   & $0.274{\scriptstyle\pm.009}$ & $0.175{\scriptstyle\pm.007}$ & $0.224{\scriptstyle\pm.006}$ & $0.247{\scriptstyle\pm.004}$ & $\mathbf{0.405}{\scriptstyle\pm.004}$ \\
Food-101    & $0.175{\scriptstyle\pm.004}$ & $0.129{\scriptstyle\pm.005}$ & $0.063{\scriptstyle\pm.002}$ & $0.204{\scriptstyle\pm.008}$ & $\mathbf{0.254}{\scriptstyle\pm.006}$ \\
Flowers-102 & $0.490{\scriptstyle\pm.016}$ & $0.375{\scriptstyle\pm.006}$ & $0.224{\scriptstyle\pm.014}$ & $0.642{\scriptstyle\pm.011}$ & $\mathbf{0.686}{\scriptstyle\pm.006}$ \\
\midrule
\multicolumn{6}{l}{\emph{classes covered ($/C$)}}\\
CIFAR-100   & $64.6{\scriptstyle\pm1.6}$ & $58.2{\scriptstyle\pm1.4}$ & $51.8{\scriptstyle\pm0.8}$ & $65.6{\scriptstyle\pm0.8}$ & $\mathbf{77.8}{\scriptstyle\pm1.4}$ \\
Food-101    & $65.8{\scriptstyle\pm0.5}$ & $53.2{\scriptstyle\pm1.5}$ & $50.6{\scriptstyle\pm1.3}$ & $62.2{\scriptstyle\pm1.7}$ & $\mathbf{71.0}{\scriptstyle\pm0.9}$ \\
Flowers-102 & $61.2{\scriptstyle\pm2.4}$ & $49.8{\scriptstyle\pm0.8}$ & $36.8{\scriptstyle\pm1.3}$ & $76.6{\scriptstyle\pm1.2}$ & $\mathbf{78.6}{\scriptstyle\pm0.7}$ \\
\bottomrule
\end{tabular}
\label{tab:unlab}
\end{table}

\paragraph{Modes $\ne$ labels: an explicit coverage or diversity correction is needed.}
The picture changes when the labels available at selection are coarser than the learner's
modes. On CIFAR-100 we let selectors use only the $20$ coarse superclasses, but
\emph{evaluate} on the $100$ fine classes (Table~\ref{tab:cf}, Figure~\ref{fig:boundary}).
Balancing the coarse labels with a \emph{scalar} cannot resolve the fine modes inside a
superclass: at the binding budget $k=C_{\rm fine}$ it reaches only $0.266$ (coarse-balanced
random) or $0.240$ (coarse-balanced prototypicality) fine accuracy, below even label-free
greedy coverage ($0.392$). But the coarse labels \emph{are} useful when paired with an
explicit within-coarse coverage step: coarse-balanced facility location --- equal quota per
superclass, chosen by facility location \emph{within} each --- reaches $0.419$, above
label-free coverage and second only to the unavailable fine-label oracle ($0.474$). So
the correction that matters is coverage/diversity, not scalar balancing. When labels
are coarser than the modes, use them as strata for coverage, not as targets for a
per-point ranking.

\begin{table}[t]\centering\small
\caption{Modes not equal labels: CIFAR-100 selected with only the $20$ coarse
superclasses, evaluated on the $100$ fine classes at $k=C_{\rm fine}$ ($3$ seeds).
Balancing the coarse labels with a \emph{scalar} is insufficient (it cannot resolve fine
modes inside a superclass), but pairing the coarse quotas with an explicit \emph{within-coarse
coverage} step (facility location) is the best practical selector --- above label-free
greedy coverage --- and trails only the unavailable fine-label oracle. The correction that
matters is coverage/diversity, not scalar balancing. Not all diversity objectives are
equally effective, though: within-coarse $k$-center improves fine-class coverage ($73.3$)
but selects boundary or atypical points and scores poorly ($0.251$), whereas facility
location better represents the dense fine modes and gives the best practical accuracy.}
\begin{tabular}{lcc}
\toprule
selector ($k=C_{\rm fine}$) & fine accuracy & fine classes ($/100$) \\
\midrule
coarse-balanced random (scalar)          & $0.266$ & $68.3$ \\
coarse-balanced prototypicality (scalar) & $0.240$ & $57.0$ \\
coarse-balanced $k$-center (coverage)    & $0.251$ & $73.3$ \\
coarse-balanced facility location (coverage) & $\mathbf{0.419}$ & $\mathbf{83.0}$ \\
label-free greedy coverage               & $0.392$ & $79.7$ \\
\emph{oracle} (fine-balanced)            & $0.474$ & $100.0$ \\
\bottomrule
\end{tabular}
\label{tab:cf}
\end{table}

\begin{figure}[t]\centering
\includegraphics[width=0.68\textwidth]{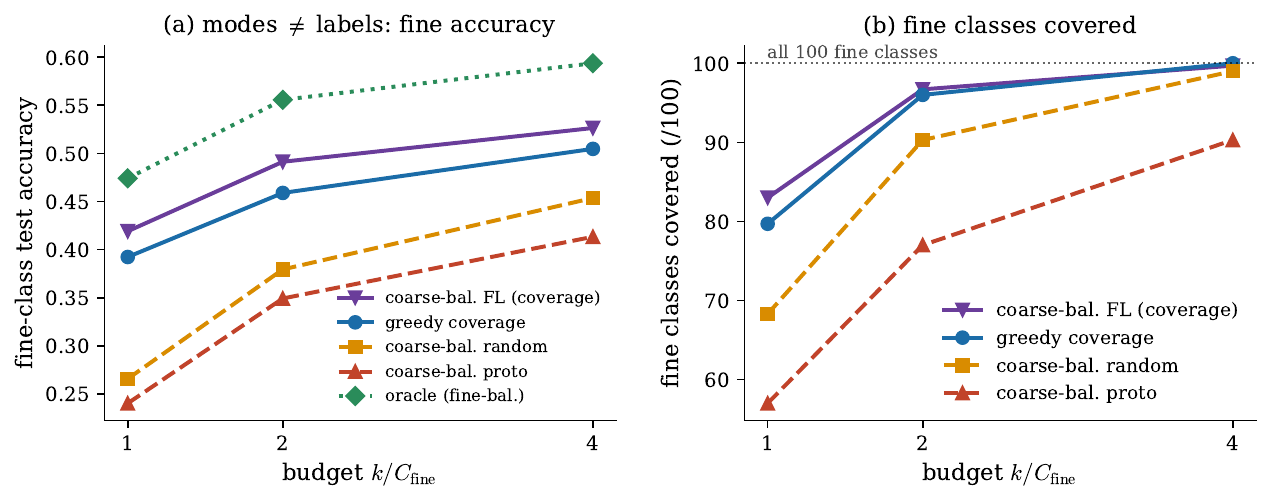}
\caption{\textbf{Modes not equal labels: coverage, not scalar balancing, is the needed
correction (CIFAR-100 coarse-to-fine).} Selectors use only the $20$ coarse superclasses,
evaluation is on the $100$ fine classes. (a) Scalar coarse balancing (orange/red) is
insufficient; \emph{coarse-balanced facility location} (purple) --- coverage within each
superclass --- is the best practical selector, above label-free greedy coverage (blue) and
second only to the unavailable fine-label oracle (green). (b) The advantage is coverage of
finer modes: the coverage selectors span the most fine classes at every budget, decisively
at the binding budget $k=C_{\rm fine}$.}
\label{fig:boundary}
\end{figure}

\subsection{Failure modes: when neither correction works}\label{sec:failures}

\noindent\emph{What this subsection shows.} Two deliberate boundary cases, both
predicted by the theory's conditions: coverage has no handle on modes the
representation does not express (Waterbirds), and coverage without a quality gate
fails when quality varies independently of coverage (corrupted pools). These delimit
the rule; they are not incidental failures.

\paragraph{Coverage needs the modes to live in the representation.} The
positive coverage results share a precondition --- the modes are geometric clusters
\emph{expressed} in the learner representation. When they are not, coverage fails. On
Waterbirds~\citep{sagawa2020} the modes are four spurious subgroups (bird $\times$
background), but a frozen generic ViT organizes points by the visually dominant
background rather than the subgroup. Facility location therefore spends its budget on
the two majority subgroups (covering only ${\sim}2$ of the $4$ even at $k=100$), and
the coverage selectors place atypical points rather than representative minority ones.
At the tight budget $k=20$, worst-group accuracy is then \emph{worst} for the coverage
methods ($0.21$--$0.23$), below random ($0.36$) and label-balancing ($0.43$), and is
recovered only by an oracle that balances the true subgroups ($0.67$). This is exactly what the theory predicts: the gap is governed by the
\emph{learner-geometry} Gram (\S\ref{sec:value}), so coverage helps only when the modes
\emph{are} that geometry's structure. A generic representation that does not express the
subgroups offers no coverage handle on them; discovering representation-hidden modes (and
worst-group robustness as the objective) is outside our scope.

\paragraph{The quality axis, controlled.} When point \emph{quality} varies independently
of coverage, pure coverage can be actively harmful --- the trade-off flagged in \S\ref{sec:limits}. We exhibit it cleanly: inject $30\%$ corrupted examples
(feature outliers with random labels) into the CIFAR-100 pool and select $k=C$
(Table~\ref{tab:quality}). Pure coverage is \emph{catastrophic} --- $k$-center and greedy
facility location fill $97$--$99\%$ of the budget with the corrupted outliers (they
maximize spread) and collapse to chance. A quality score (typicality) avoids the outliers
but concentrates. The \emph{decoupled} recipe --- restrict to the high-quality half, then
cover within it --- selects no corrupted points and wins by a wide margin, instantiating
the quality-controlled-pool-then-coverage rule of \S\ref{sec:limits}.

\begin{table}[t]\centering\small
\caption{\textbf{Quality and coverage are orthogonal: a controlled demonstration}
(CIFAR-100, $30\%$ corrupted feature-outliers, $k=C$, $5$ seeds). Pure coverage selects
almost only corrupted points and collapses to chance; quality-then-coverage avoids them
and is best --- the decoupled recipe.}
\setlength{\tabcolsep}{6pt}
\begin{tabular}{lcc}
\toprule
selector ($k=C$) & test accuracy & \% selected corrupted \\
\midrule
random & $0.156{\scriptstyle\pm.009}$ & $34\%$ \\
quality (typicality) top-$k$ & $0.226{\scriptstyle\pm.004}$ & $0\%$ \\
$k$-center (coverage) & $0.011{\scriptstyle\pm.002}$ & $99\%$ \\
greedy facility location (coverage) & $0.011{\scriptstyle\pm.003}$ & $97\%$ \\
quality-filter $+$ coverage & $\mathbf{0.387}{\scriptstyle\pm.007}$ & $0\%$ \\
\bottomrule
\end{tabular}
\label{tab:quality}
\end{table}

\subsection{The diagnostic: balance or cover, decided before selecting}\label{sec:diagnostic}

\noindent\emph{What this subsection adds.} The boundary of \S\ref{sec:boundary} is only
useful if a practitioner can locate themselves on it \emph{before} selecting. We package
the rule as a decision procedure (Algorithm~\ref{alg:rule}), define its statistics
(Definition~\ref{def:wlmm}), and validate it on every setting in the paper --- including
the settings that stress it.

\begin{algorithm}[t]
\caption{Mode-aware selection: a diagnostic decision tree (balance-or-cover)}\label{alg:rule}
\begin{algorithmic}[1]
\State \textbf{Input:} pool features $\Phi\in\mathbb R^{N\times d}$, budget $k$, optional labels $y$ with $C$ classes
\State $K\gets$ learner-geometry kernel of $\Phi$; compute redundancy diagnostics (off-diagonal Gram mass, effective rank) \Comment{is a raw top-$k$ risky?}
\If{labels unavailable}
  \State \Return learner-geometry coverage of $K$ (greedy facility location; $k$-center as a cheap variant)
\Else
  \State $a\gets\mathrm{NMI}\big(k\text{-means}(\Phi,C),\,y\big)$ \Comment{label--cluster alignment: the routing statistic}
  \State report wLMM and wSPLIT (Definition~\ref{def:wlmm}) \Comment{within-label multimodality: caution signals}
  \If{$a\ge\tau$ \textbf{(pre-fixed $\tau=0.45$)}} \Comment{labels name the modes}
    \State \Return class-balanced selection: ${\approx}\lfloor k/C\rfloor$ per class, ranked within class by the chosen score
  \Else \Comment{labels coarser than / unaligned with the modes}
    \State \Return learner-geometry coverage, stratified by the labels when available
  \EndIf
\EndIf
\end{algorithmic}
\end{algorithm}

\begin{definition}[diagnostic statistics]\label{def:wlmm}
Let $\Phi$ be the (PCA-reduced) pool features and $y$ the labels available at selection,
with $C$ label values. The \emph{routing statistic} is
$a=\mathrm{NMI}(k\text{-means}(\Phi,C),y)$, the normalized mutual information between a
$C$-way $k$-means clustering of the features and the labels. Two \emph{within-label
multimodality} statistics are reported alongside it: \emph{wLMM} (spread version), the
mean over labels of the participation ratio
$\big(\sum_\ell\lambda_\ell\big)^2/\sum_\ell\lambda_\ell^2$ of that label's feature
covariance spectrum; and \emph{wSPLIT} (cluster version), the mean over labels of the SSE
reduction of a within-label 2-means split, $1-\mathrm{SSE}_2/\mathrm{SSE}_1\in[0,1]$, over
labels with at least six points ($\mathrm{SSE}_1>0$ for every such label in our data) ---
near $0$ for a diffuse unimodal label, large when a label contains separated sub-modes.
All are computed from the frozen representation and selection-time labels only; no test
labels or accuracies enter.
\end{definition}

\noindent Effective rank and off-diagonal Gram mass are \emph{redundancy} diagnostics:
they indicate whether a raw top-$k$ ranking is likely to concentrate in a redundant
learner geometry, not the number of semantic classes in a high-class image dataset.
The balance-vs-cover branch is set by the alignment statistic $a$ alone: the released
code routes on $a\ge\tau$ with $\tau=0.45$ fixed in advance. wLMM and wSPLIT are
reported as \emph{caution signals}: they flag residual within-label structure (labels
coarser than the modes) but are not thresholded by the router. The diagnostic is a
heuristic for choosing between two \emph{safe} corrections, not an oracle. We do not
tune its threshold on test accuracy, and the experiments report the balanced and
coverage selectors \emph{directly}, so our conclusions do not depend on the routing
firing correctly. When the alignment statistic is ambiguous, we recommend reporting
\emph{both} safe selectors and avoiding a raw per-point top-$k$.

\paragraph{The diagnostic, validated where it works and where it strains.}
Algorithm~\ref{alg:rule} routes between balance and coverage without using test
accuracy: redundancy statistics flag whether a raw top-$k$ is risky, and the branch is
set by the pre-fixed alignment threshold. Table~\ref{tab:diag} reports every setting
in the paper, including the two designed to stress the router.
\emph{Correct routings:} the five modes-$=$-labels datasets have $a\ge0.545>\tau$ and
route to \textsc{balance}, matching the best selector family
(Tables~\ref{tab:fullsel} and~\ref{tab:budgets}); the no-label settings route to coverage by construction.
\emph{An honest miss:} with CIFAR-100's coarse labels on identical features, the
alignment drops $0.638\to0.466$ but lands $0.016$ above $\tau$, so the NMI router
alone would say \textsc{balance} while the correct branch is stratified coverage
(Table~\ref{tab:cf}). Any threshold in $(0.47,0.54)$ would separate every setting we
measured, but we do not move $\tau$ post hoc. Instead the case illustrates the
pre-stated protocol: a near-$\tau$ alignment, together with a wLMM jump on matched
features ($10.3\to14.2$, $+38\%$), is exactly the ambiguity signature under which one
reports \emph{both} safe selectors --- and never a raw top-$k$.
\emph{The audit row:} on Waterbirds the available $2$-class labels align well with the
geometry ($a=0.672$), and balancing them is indeed the better safe correction
(\S\ref{sec:failures}). The failure lives in modes that no statistic of features plus
available labels can see; the subgroup audit row ($a_g=0.398$) detects it when
subgroup annotations exist. The routing is thus predictive where the representation expresses the modes,
fails visibly and diagnosably where labels are marginally coarse, and delegates to
representation auditing where the modes are hidden.

\begin{table}[t]\centering\footnotesize\renewcommand{\arraystretch}{0.88}
\caption{\textbf{Diagnostic validation, including the settings where it is stressed}
(frozen ViT, $3$ seeds). Effective rank / off-diagonal Gram mass measure redundancy;
$a=$ NMI of $k$-means$(C)$ vs.\ the selection-time labels is the routing statistic
(Definition~\ref{def:wlmm}); wLMM/wSPLIT are the reported caution signals. \emph{Top:}
all five modes-$=$-labels datasets have $a\ge0.545>\tau=0.45$ and route to
\textsc{balance}, matching the best selector family (Tables~\ref{tab:fullsel}
and~\ref{tab:budgets}).
\emph{Bottom:} the stress rows. With \emph{coarse} labels on identical CIFAR-100
features, $a$ drops $0.638\to0.466$ and wLMM rises $10.3\to14.2$, but $a$ lands
$0.016$ \emph{above} the pre-fixed $\tau$, so the router \emph{misfires} (the correct
branch is stratified coverage, Table~\ref{tab:cf}); any $\tau\in(0.47,0.54)$ would
separate every setting correctly, but we do not move the threshold post hoc --- instead
the near-$\tau$ value plus the paired wLMM jump trigger the pre-stated ambiguity rule
(report both selectors). On Waterbirds the available 2-class labels align with the
geometry ($a=0.672$; balancing them is indeed the better safe correction,
\S\ref{sec:failures}), while the \emph{audit row} shows the true subgroup modes are not
expressed ($a_g=0.398$) --- no statistic of features and available labels can route
around representation-hidden modes.}
\setlength{\tabcolsep}{3pt}
\begin{tabular}{lcccccccl}
\toprule
setting & $C$ & eff.\ rank & off-diag & NMI$(C)$ & ARI$(C)$ & wLMM & wSPLIT & branch \\
\midrule
CIFAR-10    & $10$  & $23.6$ & $0.38$ & $0.749$ & $0.681$ & $15.9$ & $0.12$ & \textsc{balance} (labels $=$ modes) \\
Oxford Pets & $37$  & $26.4$ & $0.37$ & $0.841$ & $0.670$ & $12.3$ & $0.16$ & \textsc{balance} \\
CIFAR-100   & $100$ & $25.6$ & $0.38$ & $0.638$ & $0.288$ & $10.3$ & $0.17$ & \textsc{balance} \\
Flowers-102 & $102$ & $26.9$ & $0.39$ & $0.864$ & $0.616$ & $5.9$  & $0.33$ & \textsc{balance} \\
Food-101    & $101$ & $24.8$ & $0.37$ & $0.545$ & $0.189$ & $10.6$ & $0.19$ & \textsc{balance} \\
\midrule
CIFAR-100, coarse labels & $20$ & $25.6$ & $0.38$ & $0.466$ & $0.287$ & $14.2$ & $0.13$ & \textsc{balance} \emph{(misroute)} \\
\quad same features, fine labels & $100$ & $25.6$ & $0.38$ & $0.638$ & $0.288$ & $10.3$ & $0.17$ & \textsc{balance} \\
Waterbirds, labels $y$ & $2$ & $32.2$ & $0.37$ & $0.672$ & $0.807$ & $30.1$ & $0.06$ & \textsc{balance} \\
Waterbirds, subgroups $g$ (audit) & $4$ & $32.2$ & $0.37$ & $0.398$ & $0.270$ & $25.4$ & $0.07$ & modes hidden \\
\bottomrule
\end{tabular}
\label{tab:diag}
\end{table}

\subsection{Broader evidence and robustness}\label{sec:broader}

\noindent\emph{What this subsection shows.} Evidence that the boundary is a property
of the selected \emph{data} rather than of any single estimator, encoder, modality,
or hyperparameter choice --- and that the coverage proxy tracks downstream accuracy
wherever we measure it. Further robustness suites (training-dynamics valuation
baselines, more learners and encoders, text, a from-scratch pilot and a LoRA fine-tune,
full budget sweeps, and hyperparameter sensitivity) are in App.~\ref{app:robust}.

\paragraph{Beyond the ridge probe.} A natural worry is that the gap is an artifact of the
closed-form linear probe. It is not: replacing the probe with a small \emph{SGD-trained}
MLP (Adam, cross-entropy, dropout) reproduces the coverage gap almost exactly --- within
$0.02$ accuracy of the ridge probe on both datasets --- and keeps the raw leverage top-$k$
below random (Table~\ref{tab:mlp}). The interaction is a property of the selected
\emph{data}; a trained nonlinear classifier benefits from coverage just as the linear
probe does. Appendix~\ref{app:robust} pushes this much further: logistic regression and
$k$-NN learners, ResNet-50 / CLIP / MiniLM-text encoders, a from-scratch CNN trained on
raw pixels and a LoRA fine-tune of the ViT backbone (the ordering persists in both), the
causal blockness dial turned on real
features, and the occupancy formula's zero-parameter match to real coverage.

\begin{table}[t]\centering\small
\caption{\textbf{Beyond the ridge probe: the gap is a property of the data, not the
estimator.} Label-free selection at $k=C$, scored with the closed-form ridge probe
vs.\ an SGD-trained MLP ($5$ seeds, $\pm$ se on the MLP). The coverage gap
(greedy-FL $-$ random) is nearly identical under both learners, and a raw leverage
top-$k$ stays below random under both. Gaps are computed from unrounded means, so they can differ by $0.001$ from differences of the displayed values.}
\setlength{\tabcolsep}{4pt}
\begin{tabular}{ll ccccc}
\toprule
dataset & learner & random & leverage & $k$-center & greedy-FL & gap (FL$-$rand) \\
\midrule
CIFAR-100   & ridge & $0.274$ & $0.175$ & $0.247$ & $\mathbf{0.405}$ & $+0.131$ \\
            & MLP   & $0.267{\scriptstyle\pm.007}$ & $0.186{\scriptstyle\pm.010}$ & $0.200{\scriptstyle\pm.008}$ & $\mathbf{0.397}{\scriptstyle\pm.010}$ & $+0.130$ \\
\midrule
Flowers-102 & ridge & $0.490$ & $0.375$ & $0.642$ & $\mathbf{0.686}$ & $+0.195$ \\
            & MLP   & $0.473{\scriptstyle\pm.016}$ & $0.357{\scriptstyle\pm.006}$ & $0.567{\scriptstyle\pm.015}$ & $\mathbf{0.653}{\scriptstyle\pm.007}$ & $+0.180$ \\
\bottomrule
\end{tabular}
\label{tab:mlp}
\end{table}

\paragraph{Coverage predicts accuracy --- the proxy validated.} Proposition~\ref{prop:supervised}
says that under mode-aligned targets the coverage risk \emph{is} the supervised risk; the
empirical counterpart is that classes covered should predict test accuracy across
selectors of every family. Figure~\ref{fig:covacc} pools the $244$ (classes covered, test
accuracy) pairs of the main selector experiments --- six settings (five datasets and
coarse-label CIFAR-100), all selectors, and budgets $k{=}C$ to $4C$, including the
KNN-Shapley suite --- and \emph{within every setting} the two are positively
rank-correlated: Spearman $+0.60$ to $+0.94$ (mean $+0.80$). Raw scalar top-$k$ selections sit in the
low-coverage, low-accuracy corner; balanced and coverage selections sit at high coverage
and high accuracy. This is the sense in which optimizing the representation-coverage
proxy optimizes downstream supervised performance in these experiments.

\begin{figure}[t]\centering
\includegraphics[width=0.72\textwidth]{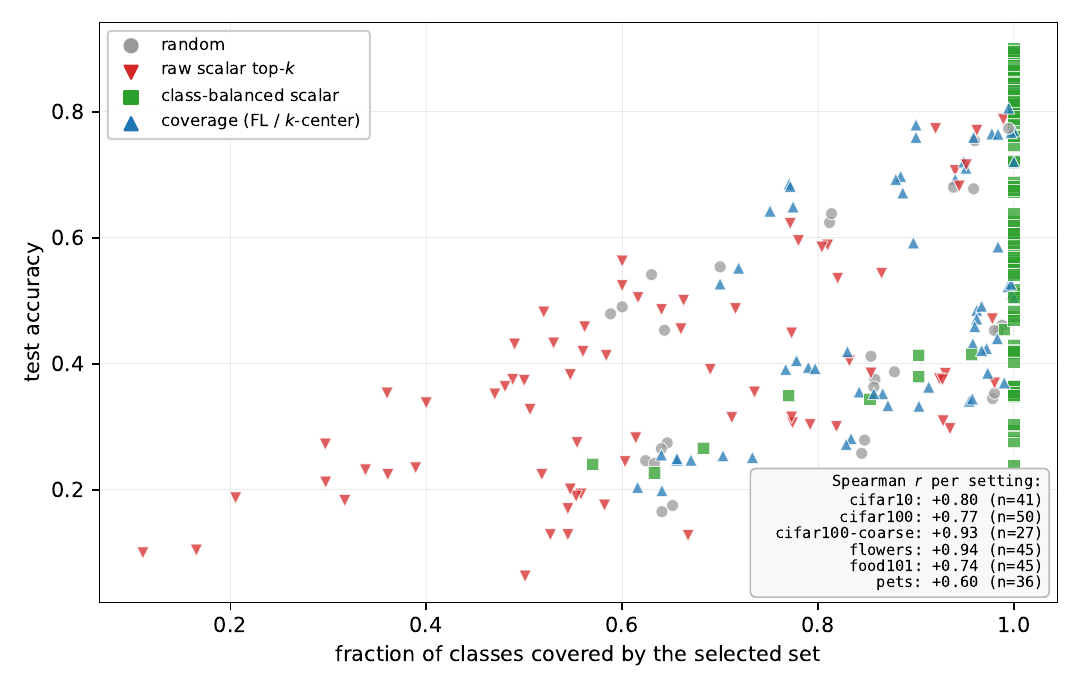}
\caption{\textbf{Coverage predicts accuracy.} The $244$ (fraction of classes covered,
test accuracy) pairs pooled across the six settings shown --- the main selector
experiments, every selector family and budget ($k=C,2C,4C$), including the
coarse-to-fine and KNN-Shapley suites; no new runs. Within each setting coverage and
accuracy are positively rank-correlated (per-setting Spearman $r$ inset, mean $+0.80$), the empirical counterpart of
Proposition~\ref{prop:supervised}: mode coverage is the operative proxy for downstream
risk. Raw scalar top-$k$ (red) concentrates --- low coverage, low accuracy; balanced
(green) and coverage (blue) selectors restore both.}
\label{fig:covacc}
\end{figure}

\paragraph{Reproducibility.} All selectors are the few-line numpy routines of
\S\ref{sec:exp} (RBF facility location, greedy projection / OMP, $k$-center, ridge
leverage, class-balanced top-$k$); the synthetic suite runs in seconds, and the
real-data results use frozen ImageNet ViT-B/16 features over $3$--$5$ seeds. Full hyperparameters --- kernel bandwidth,
the facility-location objective, prototypicality, ridge $\lambda$, the MLP architecture and
optimizer, and each dataset's pool/budget construction --- are in
Appendix~\ref{app:details}; code and exact configurations are available at
\url{https://github.com/Benjy-byj/Data-Value-Class-Balance}.

\section{Related work}

\paragraph{Per-point data valuation.} A large literature assigns each training point a
scalar worth: Data Shapley~\citep{ghorbani2019} and its efficient or reweighted variants
--- KNN-Shapley~\citep{jia2019}, Beta-Shapley~\citep{kwon2022}, Data
Banzhaf~\citep{wang2023}, LAVA~\citep{just2023}, learned valuation~\citep{yoon2020} ---
and gradient / influence attribution: influence functions~\citep{koh2017},
TracIn~\citep{pruthi2020}, datamodels~\citep{ilyas2022}, TRAK~\citep{park2023}, and
memorization estimates~\citep{feldman2020}. All return per-point numbers. We do not claim
every conceivable scalar fails, but that a fixed, mode-blind, non-adaptive per-point
ranking used as a top-$k$ selector is provably suboptimal in the redundant latent-mode
regime; our experiments instantiate this with geometry-based scores and KNN-Shapley, and
the theory covers gradient/datamodel scores when they are used as fixed mode-blind
rankings. This sharpens two Data-Shapley-specific observations: for data
\emph{selection} it can match random absent a structured utility~\citep{wang2024rethink},
and as a \emph{batch} active-learning filter it needs an explicit diversity / coverage
criterion~\citep{ghorbani2022batchal}. Our contribution is the exact occupancy gap and the
modes-vs-labels boundary that determine when balancing suffices and when explicit coverage
is needed.

\paragraph{Coresets and data pruning.} Set-level selection instead optimizes a subset:
submodular coresets~\citep{wei2015,mirzasoleiman2020,killamsetty2021}, facility-location
summarization~\citep{lin2011}, $k$-center / core-set active learning~\citep{sener2018},
and leverage / Nystr\"om sampling~\citep{drineas2005}. Difficulty-based pruning scores ---
forgetting~\citep{toneva2019}, EL2N/GraNd~\citep{paul2021}, proxy
selection~\citep{coleman2020} --- are again per-point, and pruning at scale shows that a
\emph{coverage} of easy and hard examples is what matters~\citep{sorscher2022}. We supply
the theory for \emph{why} set-level coverage is needed here, the exact scalar gap $\delta$,
and the boundary saying when the class-balance proxy suffices.

\paragraph{Information-theoretic selection.} Bayesian experimental design and active
learning select the points most informative about a quantity of
interest~\citep{chaloner1995,mackay1992}: the model parameters or latent function
(BALD and its relatives; \citealp{houlsby2011}), predictions at target inputs (expected
predictive information gain, EPIG; \citealp{bickfordsmith2023}), or, when selecting
training data, the loss on a holdout set~\citep{mindermann2022}. These are \emph{set}
objectives, and our coverage objective is one of them: for a Gaussian linear model with
isotropic prior and noiseless observations, the posterior covariance of the coefficients
after observing $S$ is $I-P_S$, so $\R(S)=\tr((I-P_S)\Sigma)$ is the integrated posterior
predictive variance over the target distribution (a V-optimality criterion) and $g_i(S)$
in Eq.~\eqref{eq:gi} is exactly the expected reduction of that variance from adding $i$
(App.~\ref{app:infotheory}). EPIG is its log-loss analogue; target-agnostic information
gain about the function, $\tfrac12\log\det(I+\sigma^{-2}K_{SS})$~\citep{krause2008,srinivas2010},
drops the $\Sigma$-weighting and scores a candidate by its own posterior variance. Our
results therefore locate information-theoretic selection rather than oppose it: used
adaptively it is set-level coverage, and used as a fixed per-point ranking it inherits the
occupancy gap --- the redundant batches of top-$k$ BALD that motivated
BatchBALD~\citep{kirsch2019} are an instance. In data attribution, \citet{tailor2026}
score training subsets by their information loss at a query --- the increase in
predictive entropy when the subset is withheld, computed with a Gaussian-process
surrogate on tangent features --- and observe that this criterion agrees with influence
scores for single examples but promotes diversity for subsets, the per-point versus
set-level distinction that our occupancy theorem makes quantitative in the latent-mode
model. Run as selectors in our protocol
(App.~\ref{app:infotheory}), greedy EPIG matches facility location (CIFAR-100: $0.411$
vs.\ $0.403$), while the same EPIG score used as a fixed top-$k$ ranking falls below
random ($0.216$ vs.\ $0.280$).

\paragraph{Data selection for LMs; marginal screening.} Curating pretraining/fine-tuning
mixtures selects or balances \emph{groups} --- importance resampling~\citep{xie2023dsir},
domain reweighting~\citep{xie2023doremi}; our boundary predicts when balancing groups is
enough (group $=$ mode) versus when label-free coverage in the model's representation is
needed. Statistically, our gap is the classical failure of univariate screening under
collinearity --- marginal / sure-independence screening~\citep{fanlv2008} vs.\ orthogonal
matching pursuit / forward selection~\citep{tropp2004} --- instantiated for data valuation,
with its exact form ($\delta$, occupancy), the phase boundary, and the counterintuitive
``more accurate score $\Rightarrow$ worse selection.''

\paragraph{Subgroup robustness.} The Waterbirds boundary connects to spurious-correlation
robustness~\citep{sagawa2020,liu2021}: there the modes are subgroups a generic representation
does not express --- precisely our learner-geometry condition.

\section{Limitations and conclusion}\label{sec:limits}
The results are for non-adaptive valuations and primarily frozen-representation
learners. The real experiments use frozen encoder features (ViT-B/16 in the main text;
the core comparisons are repeated on ResNet-50, CLIP ViT-B/32, and MiniLM text
embeddings in App.~\ref{app:robust}) and probe-family learners (ridge, an SGD-trained
MLP, logistic regression, and $k$-NN). A from-scratch CNN pilot and a LoRA fine-tune
of the selection backbone extend this: the selector ordering persists under
from-scratch training and parameter-efficient fine-tuning, respectively
(App.~\ref{app:robust}). Full-scale fine-tuning at larger budgets, LLM validation, and
tightening the constants of the general coherence-governed gap are natural next steps.

A second axis lies \emph{outside} our model: when point \emph{quality} varies
independently of coverage, a quality scalar can dominate, and pure coverage may select
unique-but-low-quality outliers. Table~\ref{tab:quality} is our formal evidence: pure
coverage selects almost only corrupted outliers and collapses, while
quality-filter-then-coverage performs best. This suggests a \emph{decoupled} recipe
--- restrict to a quality-controlled pool, then cover within it --- and formalizing
the quality--coverage trade-off is a natural extension.

The message is sharp but scoped: a point's value is a context function, not a number.
Where selection matters most --- a binding budget, a redundant learner geometry, and
non-interpolable modes --- a raw fixed per-point top-$k$ concentrates and
under-covers, and should not be used without a balancing or coverage correction. Which
correction depends on a single question: \emph{do the labels available at selection
name the learner's modes?} If yes, class-balance the scores; this recovers the missing
coverage and often beats generic set-level coverage. If no, add an explicit
learner-geometry coverage correction. And if the modes are not expressed in the
representation at all, coverage has no handle on them; representation learning or
subgroup discovery is the missing step.

\subsubsection*{Broader Impact Statement}
Subset choice shapes who and what a model represents. Our Waterbirds result shows geometric
coverage inherits the blind spots of the representation it covers: subgroups the encoder does
not express are recovered by neither coverage nor label balancing, so practitioners should
audit subgroup representation directly rather than trust diversity objectives to protect
minorities. The same diagnostics could be misused to exclude data sources; an auditable
balance-or-cover rule instead makes selection easier to inspect and contest than opaque scores.

\subsubsection*{Acknowledgments}
We thank the Action Editor and the four anonymous reviewers for careful and constructive
reviews that substantially improved this paper. The first author also thanks the mentors and
colleagues who provided guidance and feedback on this work.

\appendix
\section*{\centering Appendix}
\addcontentsline{toc}{section}{Appendix}

The appendix gives complete proofs of all results in the main text. Appendix~\ref{app:backbone}
derives the closed form \eqref{eq:gi} and proves the no-interaction geometry theorem
(Theorem~\ref{thm:nointeraction-restate}). Appendix~\ref{app:latent} proves the
latent-mode Lemmas~\ref{lem:bimodal}--\ref{lem:exch}; Appendix~\ref{app:scalars}
computes leave-one-out and Data Shapley exactly; Appendix~\ref{app:gap} proves the
interaction-gap Theorem~\ref{thm:gap} and its corollaries; and Appendix~\ref{app:soft}
gives the soft-mode (spiked-covariance) extension. Appendix~\ref{app:details} gives
experimental details (selectors, learners, features, and per-dataset configurations) and
relates facility location to the theory's selector; Appendix~\ref{app:robust} collects
additional robustness experiments; Appendix~\ref{app:num} records the numerical
verification; and Appendix~\ref{app:infotheory} reads the coverage objective as Bayesian
experimental design, runs information-theoretic selectors, and tests the
mode-aligned-target assumption directly.

\section{The coverage risk and the marginal-value closed form}\label{app:backbone}

\paragraph{Setup.} For a kernel/linear learner and \emph{point} selection, the quantity
the selected set must reconstruct is the feature second moment
$\Sigma=\E_x[\varphi\varphi^\top]\succeq0$, where $\varphi=\varphi(x)$ is the feature map.
Writing $P_S$ for the orthogonal projector onto $\mathrm{span}\{\varphi_i:i\in S\}$, the
(population) risk of the best learner supported on $S$ is the uncovered second-moment mass
\begin{equation}
\R(S)=\tr\!\big((I-P_S)\,\Sigma\big)
       =\sum_{\ell}\sigma_\ell\,\big\|(I-P_S)e_\ell\big\|^2 ,
\label{eq:riskproj}
\end{equation}
with $(\sigma_\ell,e_\ell)$ the eigenpairs of $\Sigma$. It is nonincreasing in $S$ (adding a column can only enlarge the range of $P_S$), so
$g_i(S)=\R(S)-\R(S\cup\{i\})\ge0$.

\begin{proof}[Derivation of the closed form \eqref{eq:gi}]
Let $w=\varphi_i^{\perp(S)}/\|\varphi_i^{\perp(S)}\|$ with
$\varphi_i^{\perp(S)}=(I-P_S)\varphi_i$ the component of $\varphi_i$ orthogonal to the
current span (assume $\varphi_i^{\perp(S)}\ne0$; otherwise $i$ is already in the span and
$g_i(S)=0$). Adding $\varphi_i$ extends the projector by exactly the rank-one term along
$w$:
\[
P_{S\cup\{i\}}=P_S+ww^\top .
\]
Substituting into \eqref{eq:riskproj},
\[
g_i(S)=\tr\!\big((P_{S\cup\{i\}}-P_S)\Sigma\big)=\tr\!\big(ww^\top\Sigma\big)
      =w^\top\Sigma\,w
      =\frac{\langle\varphi_i^{\perp(S)},\,\Sigma\,\varphi_i^{\perp(S)}\rangle}
             {\|\varphi_i^{\perp(S)}\|^2},
\]
which is \eqref{eq:gi}: the marginal value is the $\Sigma$-Rayleigh quotient of the
point's feature \emph{after orthogonalizing against the selected set} (a Gram--Schmidt /
partial-correlation residual).
\end{proof}

\paragraph{Theorem~\ref{thm:nointeraction-restate} (no-interaction geometry, restated).}
\emph{If $\varphi_i\perp\varphi_j$ for all $i\ne j$, then for every context $S\not\ni i$
the marginal $g_i(S)$ equals the constant $\langle\varphi_i,\Sigma\varphi_i\rangle/\|\varphi_i\|^2$,
independent of $S$. Conversely, with redundant features the marginal becomes context-dependent
through two distinct mechanisms: \emph{(i) directional rotation} of the residual when its
direction changes and $\Sigma$ is anisotropic on the shared directions, and \emph{(ii)
rank collapse} when a context drives the residual to zero, both controlled by the
off-diagonal mass of $G=\Phi\Phi^\top$.}

\begin{proof}
If the features are mutually orthogonal then for any $S\not\ni i$ we have
$P_S\varphi_i=0$, hence $\varphi_i^{\perp(S)}=\varphi_i$ and \eqref{eq:gi} gives the
stated constant for every such $S$; the interaction $\mathrm{Int}_i=0$. Conversely, suppose
$\varphi_i$ has a nonzero component along $\mathrm{span}\{\varphi_j:j\in S\}$ for some $S$.
\emph{(i)} If the residual stays nonzero, its unit direction
$w_i(S)=\varphi_i^{\perp(S)}/\|\varphi_i^{\perp(S)}\|$ varies with $S$, and since
$g_i(S)=w_i(S)^\top\Sigma\,w_i(S)$ is a Rayleigh quotient, an anisotropic $\Sigma$ makes it
vary as $w_i(S)$ rotates; Proposition~\ref{prop:intbound}(i) bounds this by
$(\lambda_{\max}-\lambda_{\min})$ times the residual-direction diameter. \emph{(ii)} If some
$S_0$ makes $\varphi_i^{\perp(S_0)}=0$, then $g_i(S_0)=0$ regardless of $\Sigma$
(Proposition~\ref{prop:intbound}(ii)). The latent-mode instance of
Appendix~\ref{app:latent} realizes mechanism~(ii) in the extreme, where the range is
$\mathrm{Int}_i=\rho_{j(i)}=\Omega(1)$.
\end{proof}

This is the classical statement that \emph{marginal (univariate) screening differs from
forward selection exactly by the redundancy that orthogonalization removes}; here applied
to data valuation. The off-diagonal Gram is the \emph{contribution covariance} referred to
in the main text.

\paragraph{Proof of Proposition~\ref{prop:supervised} (mode-aligned targets).}
Write $U_{\mathrm{cov}}=\{j:S\cap P_j\ne\emptyset\}$ and
$U_{\mathrm{unc}}=\{j:S\cap P_j=\emptyset\}$. Since $\xi_i\perp\mathrm{span}\{u_\ell\}$ by
construction, every selected feature satisfies
$\langle u_j,\varphi_i\rangle=\langle u_j,u_{j(i)}\rangle=\one[j=j(i)]$; hence for an
\emph{uncovered} mode $j$, $u_j\perp\varphi_i$ for every $i\in S$ \emph{exactly}, so no
$\theta\in\mathrm{span}\{\varphi_i:i\in S\}$ has a $u_j$-component. The excess risk of a
linear predictor $\theta$ is
$\E(\theta^\top\varphi-\theta_\star^\top\varphi)^2=\|\theta-\theta_\star\|_\Sigma^2$ with
$\|v\|_\Sigma^2=v^\top\Sigma v$.
\emph{Lower bound:} for any feasible $\theta$, the difference $\theta-\theta_\star$ has
$u_j$-component exactly $-\beta_j$ for each $j\in U_{\mathrm{unc}}$, so
$\|\theta-\theta_\star\|_\Sigma^2\ge\sum_{j\in U_{\mathrm{unc}}}\beta_j^2\pi_j$.
\emph{Upper bound:} for each covered mode $j$ pick one selected point
$\varphi_{i_j}=u_j+\eta\,\xi_{i_j}$ and set
$\theta=\sum_{j\in U_{\mathrm{cov}}}\beta_j\varphi_{i_j}$. Then
$\theta-\theta_\star=-\sum_{j\in U_{\mathrm{unc}}}\beta_ju_j
+\eta\sum_{j\in U_{\mathrm{cov}}}\beta_j\xi_{i_j}$; the two parts are
$\Sigma$-orthogonal (mode subspace vs.\ noise subspace), and since $\Sigma v=\eta^2v$ on
the noise subspace,
$\|\theta-\theta_\star\|_\Sigma^2
=\sum_{j\in U_{\mathrm{unc}}}\beta_j^2\pi_j
+\eta^4\big\|\sum_{j\in U_{\mathrm{cov}}}\beta_j\xi_{i_j}\big\|^2
=\sum_{j\in U_{\mathrm{unc}}}\beta_j^2\pi_j+O(\eta^4)$.
The minimizer over the span is squeezed between the two, giving the claim; at $\eta=0$
the bounds coincide and the identity is exact.
For the classification statement, labels are constant on modes; a covered mode contains a
training point whose feature matches the mode direction up to $O(\eta)$, so the argmax
readout of the fitted one-hot regression predicts its label correctly for small $\eta$,
while an uncovered mode's direction is orthogonal to the span, leaving its logits at the
prior --- chance level. The $0$--$1$ risk therefore decomposes as
$\sum_{j\in U_{\mathrm{unc}}}\Pr[j]\,(1-1/C)+O(\eta)$, the coverage risk with
$\rho_j=\Pr[j](1-1/C)$. \qed

\section{The latent-mode instance: Lemmas~\ref{lem:bimodal}--\ref{lem:exch}}\label{app:latent}

Recall the instance: $P=\bigsqcup_{j=1}^r P_j$ with $r$ modes of size $m=N/r$, a block
learner geometry in which covering $\ge1$ point of mode $j$ drives that mode's residual to
$0$, and an uncovered mode costs $\rho_j>0$. Equation~\eqref{eq:riskproj} then reduces to
\begin{equation}
\R(S)=\sum_{j:\,S\cap P_j=\emptyset}\rho_j .
\label{eq:blockrisk}
\end{equation}

\paragraph{Lemma~\ref{lem:bimodal} (all-or-nothing marginal, restated).}
\emph{For $i\in P_j$, $\;g_i(S)=\rho_j\,\one[S\cap P_j=\emptyset]\in\{0,\rho_j\}$, so
$\mathrm{Int}_i=\rho_j$. Any scalar $v_i\in[0,\rho_j]$ disagrees with the true marginal on
an $\Omega(1)$ fraction of contexts.}

\begin{proof}
Take $i\in P_j$ and any context $S$ with $i\notin S$. If $S\cap P_j=\emptyset$, mode $j$ is
uncovered in $S$ but covered in $S\cup\{i\}$, so by \eqref{eq:blockrisk}
$g_i(S)=\R(S)-\R(S\cup\{i\})=\rho_j$. If $S\cap P_j\ne\emptyset$, mode $j$ is already
covered and remains so, the term $\rho_j$ is absent from both risks, and no other mode's
coverage changes by adding a mode-$j$ point, so $g_i(S)=0$. Hence
$g_i(S)=\rho_j\one[S\cap P_j=\emptyset]$ and $\mathrm{Int}_i=\rho_j-0=\rho_j$. Finally, a
fixed scalar $v_i$ predicts a single value, whereas the true marginal takes both $0$ and
$\rho_j$ on positive-probability context families (e.g.\ random $S$ of size $k$ with
$k<r$ has $\Pr[S\cap P_j=\emptyset]=\Omega(1)$ and $\Pr[S\cap P_j\ne\emptyset]=\Omega(1)$);
so $|v_i-g_i(S)|\ge\rho_j/2$ on at least one of the two families, an $\Omega(1)$ fraction.
\end{proof}

\paragraph{Lemma~\ref{lem:exch} (exchangeability collapses scores, restated).}
\emph{Every valuation that is a symmetric functional of the training distribution is
constant on each mode: $v_i=c_j$ for all $i\in P_j$.}

\begin{proof}
Points within a mode are distributionally identical (exchangeable): there is a
measure-preserving relabeling of the pool that swaps any $i,i'\in P_j$ and fixes all
other points and the distribution $D$. Leave-one-out, Data Shapley, leverage and
probe-gap are all symmetric functionals --- their definitions are invariant to the names of
the points and depend only on the multiset of features/labels. Applying the swap leaves
the functional unchanged, so $v_i=v_{i'}$. Hence $v$ is constant on $P_j$.
\end{proof}

\section{The two canonical scalars, computed exactly}\label{app:scalars}

\paragraph{Leave-one-out / influence.} By definition
$v_i^{\mathrm{LOO}}=g_i(P\setminus\{i\})=\R(P\setminus\{i\})-\R(P)$. At full data every
mode is covered, so $\R(P)=0$. For $i\in P_j$ with $m\ge2$, the set $P\setminus\{i\}$ still
contains the other $m-1\ge1$ points of mode $j$, so mode $j$ stays covered and
$\R(P\setminus\{i\})=0$. Therefore $v_i^{\mathrm{LOO}}=0$ for \emph{every} point: in the
redundant regime each point is individually removable and influence sees no value, leaving
top-$k$ an arbitrary tie-break.

\paragraph{Data Shapley.} For a random permutation $\pi$ and $i\in P_j$, the marginal
$g_i(S^\pi_{<i})$ equals $\rho_j$ if no point of mode $j$ precedes $i$ in $\pi$ (mode $j$
still uncovered when $i$ arrives) and $0$ otherwise (Lemma~\ref{lem:bimodal}). Hence
\[
v_i^{\mathrm{Sh}}=\E_\pi\,g_i(S^\pi_{<i})
   =\rho_j\,\Pr_\pi[\,i\text{ is the first point of mode }j\,]
   =\frac{\rho_j}{m_j},
\]
because the $m_j$ points of mode $j$ are equally likely to be the first of their mode in a
uniform random order. For equal modes $v^{\mathrm{Sh}}$ is constant over \emph{all} points;
for unequal modes it is mode-constant (up-weighting rare modes) but still mode-measurable,
hence subject to Lemma~\ref{lem:exch} and the gap below.

\section{Proof of the interaction gap}\label{app:gap}

Throughout, $\R(S)=\rho\,(r-N_{\mathrm{cov}}(S))$ with
$N_{\mathrm{cov}}(S)=\#\{j:S\cap P_j\ne\emptyset\}$ (Eq.~\eqref{eq:blockrisk}), so
comparing selectors reduces to comparing expected covered modes. Greedy facility
location gains $\rho$ from a fresh mode and $0$ from a covered one
(Lemma~\ref{lem:bimodal}), so it covers a new mode each step while one remains:
$N_{\mathrm{cov}}(S_{\mathrm{FL}})=\min(k,r)$, exactly optimal (a fortiori
$(1-1/e)$-optimal by submodularity~\citep{nemhauser1978}). We use a uniform
tie-breaking rule throughout (Theorem~\ref{thm:gap}'s hypothesis).

\paragraph{Equality for symmetric valuations (Theorem~\ref{thm:gap}).}
With \emph{equal} modes the whole pool is exchangeable: any relabeling permuting points
within \emph{and across} modes preserves the data distribution. A data-dependent
symmetric valuation is invariant under such relabelings, so it assigns one common value
to all $N=rm$ points. Its top-$k$ set, ties broken uniformly, is therefore a uniformly
random $k$-subset of $P$. A fixed mode $j$ is uncovered iff none of its $m$ points is
drawn, whence
\[
\Pr[\text{mode }j\text{ missed}]=\frac{\binom{N-m}{k}}{\binom{N}{k}}
   =\frac{\binom{(r-1)m}{k}}{\binom{rm}{k}},\qquad
\E[N_{\mathrm{cov}}]=r\Big(1-\tfrac{\binom{(r-1)m}{k}}{\binom{rm}{k}}\Big)=m_\star(k),
\]
exact for finite $N$; as $N\to\infty$ with $m/N=1/r$,
$\binom{(r-1)m}{k}/\binom{rm}{k}\to(1-1/r)^k$, recovering $m_\star(k)\to r(1-(1-1/r)^k)$.
Subtracting from greedy gives
$\E[\R(S_v)]-\R(S_{\mathrm{FL}})=\rho\,(\min(k,r)-m_\star(k))=\rho\,\delta_N(k)$.
\emph{Positivity range.} At $k=1$, $m_\star(1)=1$: one draw always covers one fresh mode,
so $\delta_N(1)=0$. For $2\le k\le\min(r,N-m)$ (with $m\ge2$), the second draw repeats
the first draw's mode with probability $(m-1)/(N-1)>0$, so $\E[N_{\mathrm{cov}}]<k$ and
$\delta_N(k)>0$. For $r<k\le N-m$, a fixed mode is missed with probability
$\binom{N-m}{k}/\binom{N}{k}>0$, so $m_\star(k)<r=\min(k,r)$ and $\delta_N(k)>0$. For
$k>N-m$, every $k$-subset intersects every mode ($k>N-m$ leaves fewer than $m$ points
unselected), so $m_\star(k)=r$ and $\delta_N(k)=0$. If $m=1$ then every point is its own
mode and any $k$-subset covers $k$ modes, so $\delta_N\equiv0$. \qed

\paragraph{Upper bound for non-coordinated scores (Proposition~\ref{prop:general},
equal modes of size $m$).} Let $v_i=c_{j(i)}+\varepsilon_i$ be non-coordinated within
modes (Definition~\ref{def:noncoord}). (Equal sizes are used in the pooling step below;
Remark~\ref{rem:unequal} shows the bound can fail for unequal modes; and
Remark~\ref{rem:coord} shows some such condition is \emph{necessary} --- jointly
coordinated within-mode randomization can beat occupancy.)
\emph{Condition on the realization of the mode-level vector $(c_1,\dots,c_r)$.} Given
$(c_j)$, all points of a mode share the same mode-level value, so the within-mode
ranking is by the conditionally i.i.d.\ $\varepsilon$ alone: conditioned additionally on
the per-mode counts $s_j=|S_v\cap P_j|$ with $\sum_j s_j=k$, the chosen $s_j$ points are
uniform inside $P_j$, and mode $j$ is covered iff $s_j\ge1$; thus
$N_{\mathrm{cov}}=\sum_j\one[s_j\ge1]$ is a function of the count vector
$s=(s_1,\dots,s_r)$ alone, whose law given $(c_j)$ is that of a deterministic tiered
mode-constant score with uniform tie-breaking. It therefore suffices to show, for every
\emph{fixed} tier vector $c$, that the uniform score (whose $s$ is multivariate
hypergeometric, the Theorem~\ref{thm:gap} case) maximizes
$\E[N_{\mathrm{cov}}\,|\,c]$; averaging over $c$ then gives the proposition.

\begin{lemma}[pooling does not decrease coverage]\label{lem:pool}
Fix all per-mode counts except two modes $a,b$ with $s_a+s_b=t$. \emph{Pooling} the pair
--- redrawing the $t$ points uniformly without replacement from the $2m$ points of
$P_a\cup P_b$ --- gives the pair expected coverage $2-2\binom{m}{t}/\binom{2m}{t}$. When
the original allocation is \emph{concentrated} ($s_a\in\{0,t\}$, so the pair covers $1$),
pooling does not decrease coverage:
\[
2-2\,\frac{\binom{m}{t}}{\binom{2m}{t}}\ \ge\ 1\qquad(t\ge1),
\]
with strict inequality for $t\ge2$.
\end{lemma}
\begin{proof}
After pooling, a given mode of the pair is missed iff all $t$ points fall in the other
mode's $m$ slots, probability $\binom{m}{t}/\binom{2m}{t}$; summing the complement over
the two modes gives the stated pair coverage. The inequality is
$\binom{m}{t}/\binom{2m}{t}\le\tfrac12$, i.e.\ the $t$ draws are at least as likely to
touch both halves as to fall in one; for $t\ge2$ the event ``all in one half'' has
probability $<\tfrac12$ strictly (e.g.\ $\binom{m}{2}/\binom{2m}{2}=\tfrac{m-1}{2(2m-1)}<\tfrac12$).
\end{proof}

\emph{From pairwise to global.} Fix the tier vector $c$. The top-$k$ takes whole
value-tiers and a uniform subset of the single active tier, so within a tier the counts
are already exchangeable and every \emph{split} pair (one mode fully taken, the other
untaken) is concentrated --- exactly the case of Lemma~\ref{lem:pool}. Repeatedly pooling split pairs
weakly increases $\E[N_{\mathrm{cov}}]$ and converges to the fully exchangeable law, the
uniform $k$-subset, with $\E[N_{\mathrm{cov}}]=m_\star(k)$ (Theorem~\ref{thm:gap}). Hence
$\E[N_{\mathrm{cov}}(S_v)]\le m_\star(k)$, with equality iff the mode-values are equal ---
for $2\le k\le N-m$; at $k=1$ every allocation covers exactly one mode, so every
mode-constant score attains the bound. This is the majorization statement that
$\sum_j\one[s_j=0]$ is Schur-convex in the counts and the uniform allocation is majorized
by every tiered one~\citep{marshall2011}. Subtracting from greedy gives
$\E[\R(S_v)]-\R(S_{\mathrm{FL}})\ge\rho\,\delta_N(k)$. \qed

\paragraph{Corollary (mode signal can help or hurt).}
For equal modes, $\E[N_{\mathrm{cov}}]$ is maximized at zero mode signal and decreases as
the signal concentrates the budget within a few modes, so a concentrating score can fall
\emph{below} uniform random at a binding budget. With \emph{unequal} sizes the effect can
go either way (Remark~\ref{rem:unequal}): a score whose high values spread budget across
\emph{rare} modes can cover \emph{more} than uniform, whereas one that concentrates within
a few large modes covers fewer. The unconditional lower bound is therefore specific to
equal-size modes; in general, whether mode signal helps or hurts coverage depends on
whether it spreads the budget across modes or concentrates it.

\paragraph{Corollary (phase boundary, restated and proved).}
\emph{(a) Exact positivity range.} By the positivity analysis above, $\rho\,\delta_N(k)>0$
iff $\rho>0$, $m\ge2$, and $2\le k\le N-m$; if any condition fails, the scalar top-$k$ is
optimal in the model.
\emph{(b) Magnitude.} Work in the $N\to\infty$ limit
$\delta(k)=\min(k,r)-r(1-(1-1/r)^k)$, where the $k$ uniform draws become independent.
\emph{Small budgets:} for $k\le r$, writing $N_{\mathrm{cov}}=k-\#\text{collisions}$ and
union-bounding the $\binom{k}{2}$ pairwise collision events (each of probability $1/r$)
gives $k-m_\star(k)\le\binom{k}{2}/r$, so the per-budget regret is
$\delta(k)/k\le(k-1)/(2r)=O(k/r)$, vanishing when $k\ll r$.
\emph{Binding budgets:} at $k=\alpha r$ with $\alpha\in(0,1]$ fixed,
$\delta(k)=r\big(\alpha-1+(1-1/r)^{\alpha r}\big)\to r(\alpha-1+e^{-\alpha})$, which is
$\Theta(r)$ for fixed $\alpha$ (at $\alpha=1$: $\delta(r)=r/e+O(1)$) --- a constant
fraction of the total risk $\rho r$.
\emph{Loose budgets:} for $k\ge r$, $\delta(k)=r(1-1/r)^k\le re^{-k/r}$, decaying
exponentially in $k/r$.
Together these give the $\delta$-shape observed empirically: the gap is a
constant-fraction excess risk iff the budget binds ($k=\Theta(r)$), the task is hard
($\rho>0$), and the learner geometry is redundant ($m\ge2$, i.e.\ $r\ll N$ with
off-diagonal Gram mass). Condition (iii) is redundancy of the \emph{learner} geometry,
not non-locality of the representation.

\section{Soft modes: the spiked-covariance extension (asymptotic)}\label{app:soft}

This appendix is an \emph{asymptotic} (small-$\eta$) extension: we state explicit
assumptions, prove the leading-order behavior with existence-form constants, and do
not optimize those constants.

\paragraph{Assumptions.}
\emph{(S1) Spikes:} $r$ orthonormal directions $u_1,\dots,u_r$; a point of mode $j$ is
$\varphi_i=u_{j(i)}+\eta\,\xi_i$.
\emph{(S2) Masses:} the population second moment is
\[
\Sigma=\sum_{j=1}^r\pi_j\,u_ju_j^\top+\eta^2P_\perp,
\]
with $P_\perp$ the projector onto the orthogonal complement of
$\mathrm{span}\{u_\ell\}$; for the equal-mode occupancy statement we take
\emph{equal} masses $\pi_j\equiv\pi=:\rho$ (unequal masses give per-mode costs
$\rho_j=\pi_j$ and fall under Remark~\ref{rem:unequal}).
\emph{(S3) Noise:} the $\xi_i$ are i.i.d., mean-zero, supported on
$\mathrm{range}(P_\perp)$ (exactly orthogonal to every spike), and bounded,
$\|\xi_i\|\le B$ almost surely.
All constants below depend only on $(r,m,k,B)$.

\paragraph{Soft Lemma~1 (leading order, constants not optimized).} \emph{Under
(S1)--(S3) there exist $\eta_0>0$ and $c<\infty$ depending only on $(r,m,k,B)$ such
that for all $\eta\le\eta_0$ and all contexts $S$ with $|S|\le k$:
if mode $j(i)$ is uncovered by $S$ then $g_i(S)\ge\pi_{j(i)}-c\,\eta^2$, and if it is
covered then $g_i(S)\le c\,\eta^2$. Informally,
$g_i(S)=\pi_{j(i)}\,\one[\text{mode }j(i)\text{ uncovered}]+O(\eta^2)$.}

\begin{proof}
\emph{Uncovered mode.} If no selected point shares mode $j=j(i)$, then $u_j\notin
\mathrm{span}\{\varphi_\ell:\ell\in S\}$ up to $O(\eta)$, so $\varphi_i^{\perp(S)}=\varphi_i+O(\eta)$.
Using $\Sigma\varphi_i=\pi_j u_j(u_j^\top\varphi_i)+\eta^2P_\perp\varphi_i
=\pi_j u_j+\eta^3\xi_i$ and $\|\varphi_i\|^2=1+\eta^2\|\xi_i\|^2$,
\[
g_i(S)=\frac{\langle\varphi_i,\Sigma\varphi_i\rangle}{\|\varphi_i\|^2}+O(\eta^2)
      =\frac{\pi_j+O(\eta^4)}{1+O(\eta^2)}+O(\eta^2)=\pi_{j}+O(\eta^2).
\]
(The apparent first-order correction from
$\varphi_i^{\perp(S)}=\varphi_i+O(\eta)$ lies in the noise subspace orthogonal to the mode
directions; since $\Sigma$ has no first-order cross term between the mode subspace and the
noise subspace, its contribution to the Rayleigh quotient is $O(\eta^2)$.) \emph{Covered mode.} If some selected point shares mode
$j$, then $u_j\in\mathrm{span}$ up to $O(\eta)$, so $\varphi_i^{\perp(S)}=\eta\,\xi_i^{\perp}+O(\eta^2)$
is pure noise of norm $O(\eta)$. Since $\xi_i\in\mathrm{range}(P_\perp)$,
$\Sigma\varphi_i^{\perp(S)}=\eta^2\varphi_i^{\perp(S)}+O(\eta^3)$, whence the Rayleigh
quotient is $g_i(S)=\eta^2+O(\eta^3)=O(\eta^2)$.
\end{proof}

The proof tracks the perturbation only to leading order; each $O(\eta^2)$ hides a
constant bounded by a polynomial in $B$ and the number of selected points, which (S3)'s
boundedness makes uniform over contexts $|S|\le k$ --- this is where $c$ and $\eta_0$
come from, and we do not compute them sharply. As $\eta\to0$ this recovers the
$\{0,\pi_j\}$ dichotomy of Lemma~\ref{lem:bimodal}, and the interaction
$\mathrm{range}_S\,g_i\to\pi_{j(i)}$ stays $\Omega(1)$.

\paragraph{Soft theorem (asymptotic).} Under (S1)--(S3) with equal masses
$\pi_j\equiv\pi=\rho$, within-mode exchangeability (Lemma~\ref{lem:exch}) is unchanged,
so the occupancy argument of Appendix~\ref{app:gap} applies to the leading term; the
covered-mode residuals contribute at most $c\,\eta^2k$ to either risk, with $c$ as in
Soft Lemma~1:
\[
\E[\R(S_v)]-\R(S_{\mathrm{FL}})\ \ge\ \rho\,\delta_N(k)-c\,\eta^2k
\quad\text{for all }\eta\le\eta_0,\ 2\le k\le N-m,
\]
which is positive once $\eta^2<\rho\,\delta_N(k)/(ck)$. We present this as an
asymptotic extension of the exact hard-mode theorem, not a sharp finite-$\eta$ result.

\paragraph{Spectral reading.} The $\pi_j$ are the spike eigenvalues of $\Sigma$, $r$ is the
effective rank, and $\eta^2$ is the tail. The gap is large for a \emph{peaked} spectrum
(low effective rank $\Rightarrow$ redundant, large off-diagonal Gram) and vanishes as the
spectrum \emph{flattens} (high effective rank $\Rightarrow$ near-orthogonal points,
diagonal Gram) --- the controlled law of the noise sweep
(Table~\ref{tab:C}). For general $\Sigma$ the problem reduces to the classical
marginal-screening-versus-orthogonal-matching-pursuit gap of sparse approximation, governed
by the Gram coherence $\mu=\max_{i\ne j}|G_{ij}|$; the spiked model is the clean computable
instance.

\section{Experimental details}\label{app:details}

\textbf{Selectors.} The kernel is RBF with the median-heuristic bandwidth
$\sigma^2=\mathrm{median}_{i\ne j}\|\varphi_i-\varphi_j\|^2$. \emph{Greedy facility
location} maximizes $\sum_j\max_{i\in S}K_{ji}$ by the standard greedy step; \emph{ridge
leverage} is $\tau_i=[K(K+\lambda I)^{-1}]_{ii}$ with $\lambda=0.1$; \emph{$k$-center} is
farthest-first from the medoid. \emph{Prototypicality} is $-\|\varphi_i-\bar\varphi_{c(i)}\|$
(distance to the class mean) and \emph{global typicality} is $-\|\varphi_i-\bar\varphi\|$.
Class-balanced selectors take $\lfloor k/C\rfloor$ points per class, ranking within a class
by the scalar and filling any shortfall at random.

\textbf{KNN-Shapley.} We compute exact KNN-Shapley~\citep{jia2019} on the frozen ViT
feature space after PCA projection, with $10$ nearest neighbors and Euclidean distance.
The valuation set is a held-out validation split drawn from the training pool ($1500$
examples per seed), disjoint from both the selected training subset and the final test
set; the scores are computed once per seed before selection. The same scores are used
two ways --- raw top-$k$, and class-balanced (labels only impose equal per-class quotas,
with points ranked within class by their KNN-Shapley value) --- so the comparison isolates
the \emph{selection rule}: the per-point values are unchanged while the raw top-$k$
procedure is replaced by a coverage-preserving quota.

\textbf{Learners.} The probe is one-hot ridge regression,
$W=(\Phi^\top\Phi+\lambda I)^{-1}\Phi^\top Y$ with $\lambda=0.1$, predicting $\arg\max$.
The MLP (Table~\ref{tab:mlp}) is $64\!\to\!256\!\to\!C$ with ReLU and dropout $0.3$,
trained by Adam (learning rate $10^{-3}$, weight decay $10^{-4}$) for $300$ full-batch steps with
early stopping on a $10\%$ held-out split of the selected set.

\textbf{Features and data.} Images are resized to $224^2$, normalized with ImageNet
statistics, and embedded by a frozen timm ViT-B/16 (\texttt{vit\_base\_patch16\_224}); the
embeddings are PCA-projected to $64$ dimensions (fit on the pool). Each seed draws a fresh
pool of $N\le5000$ training points and as many test points; budgets are $k\in\{C,2C,4C\}$,
binding at $k=C$. CIFAR-10/100, Oxford Pets, Flowers-102, and Food-101 use their standard
splits; Waterbirds uses the WILDS subgroup metadata. We report means over $3$--$5$ seeds
with standard errors; gaps and differences are computed from unrounded means and then
rounded. The synthetic suite uses $C=10$ classes with $3$ input-distant modes
each, a non-degenerate embedding, and the same ridge probe.

\textbf{Facility location vs.\ the theory's greedy coverage, and scalability.} The
theory's greedy coverage adds, at each step, a point maximizing the marginal
\eqref{eq:gi} --- greedy trace-reduction of $\tr((I-P_S)\Sigma)$, i.e.\ orthogonal
matching pursuit on the feature second moment. RBF facility location is the empirical
selector. The relationship between the two is settled in three steps rather than
asserted: the next lemma shows they \emph{coincide exactly} in the analyzed block
instance; the greedy-proj column of Table~\ref{tab:B} confirms the match empirically at
binding budgets in the near-block synthetic regime ($0.889$ vs $0.895$ at $k=C$); and
Table~\ref{tab:omp} runs both on real features, where they \emph{diverge} --- exactly
where the lemma's soft-mode condition fails --- which delimits what the theory's literal
objective does and does not capture on real data.

\begin{lemma}[facility location realizes projection greedy in the block model]\label{lem:flproj}
In the hard latent-mode instance ($\varphi_i=u_{j(i)}$ with orthonormal mode directions
$u_j$, equal sizes $m$, equal costs), consider any kernel of the form
$K_{xy}=\kappa(\|\varphi_x-\varphi_y\|)$ with $\kappa$ strictly decreasing (RBF is such a
kernel): $K_{xy}=\kappa(0)$ iff $x,y$ share a mode and $K_{xy}=\kappa(\sqrt2)<\kappa(0)$
otherwise. Then for every step $t\le r$, greedy facility location adds a point of a
previously \emph{uncovered} mode (its gain is $m\,(\kappa(0)-\kappa(\sqrt2))>0$ versus
$0$ for any duplicate), exactly as greedy projection does (marginal $\rho_j$ vs $0$,
Lemma~\ref{lem:bimodal}); both therefore cover $\min(k,r)$ modes and achieve identical
risk at every budget. With soft modes ($\eta>0$), the fresh-mode facility-location gain
is $m(\kappa(0)-\kappa(\sqrt2))+O(L\eta)$ and the duplicate gain $O(L\eta)$ ($L$ the
Lipschitz constant of $\kappa$), so the correspondence persists whenever
$m(\kappa(0)-\kappa(\sqrt2))\gg L\eta$; outside this regime the two objectives can
diverge, and Table~\ref{tab:B} shows where they do (loose budgets, where the theory
predicts no gap).
\end{lemma}
\begin{proof}
Facility location's objective is $F(S)=\sum_{x\in P}\max_{i\in S}K_{xi}$. In the hard
instance all points of a mode are identical, so after any selection the running max for a
point $x$ is $\kappa(0)$ if $x$'s mode is covered and $\kappa(\sqrt2)$ if some other mode
is covered (or $0$ if $S=\emptyset$; the first step then gains
$m\kappa(0)+(N-m)\kappa(\sqrt2)$, maximized by any point). Adding a duplicate of a
covered mode changes no point's max: its mode-mates already sit at $\kappa(0)$ and every
other point's max is $\ge\kappa(\sqrt2)$, which equals the new candidate's kernel value
--- gain $0$. Adding a first representative of an uncovered mode raises exactly its $m$
mode-mates from $\kappa(\sqrt2)$ to $\kappa(0)$ --- gain $m(\kappa(0)-\kappa(\sqrt2))>0$.
Greedy facility location therefore picks a fresh mode while one remains, matching greedy
projection's fresh-mode picks (Lemma~\ref{lem:bimodal}); with equal sizes and costs the
covered-mode counts, and hence the risks, coincide at every step. The soft-mode
perturbation follows since every kernel entry moves by at most $L\cdot O(\eta)$ and each
gain aggregates $\le N$ entries.
\end{proof}

\noindent Numerically (\texttt{experiments\_omp.py}, $20$ random instances, $r=10$,
$m=5$): at $\eta=0$ the two selectors cover identically $\min(k,r)$ modes in all $60$
checks; at $\eta=0.15$ they agree on $98\%$ of checks with coverage fractions $0.998$
(projection) vs $1.000$ (facility location) --- the correspondence degrades only as the
soft-mode condition weakens, as the lemma predicts.

\noindent The \emph{recovered} modes of the empirical selector depend on the kernel
bandwidth, the PCA dimension, and feature normalization (the median heuristic and
PCA-$64$ above worked across all datasets without tuning; the synthetic redundancy sweep,
Table~\ref{tab:C}, isolates the bandwidth/effective-rank dependence). Facility location
needs the $N\times N$ Gram, i.e.\ $O(N^2)$ time and memory, which is the relevant cost at
the binding budgets we study ($k$ in the hundreds, pools of thousands); greedy projection
as implemented is $O(Nkd+Nd^2)$. For larger pools,
standard approximations include lazy (accelerated) greedy, $k$-NN graph sparsification,
approximate-neighbor kernels, and stochastic-greedy or streaming submodular maximization;
these trade exactness for scalability, and evaluating them at larger scale is orthogonal to
the balance-vs-cover boundary studied here.

\begin{table}[h]\centering\small
\caption{\textbf{The theory's literal selector vs.\ the empirical one, on real
features} (label-free selection, frozen ViT, $k=C$, $3$ seeds, accuracy / classes
covered). Unlike in the block regime (Lemma~\ref{lem:flproj}, Table~\ref{tab:B}), raw
greedy projection \emph{underperforms} facility location on real features and behaves
like the $k$-center family: real feature spectra are heavy-tailed, so trace-reduction
spends its budget on spectral extremes (atypical directions), while the modes that
carry the labels are \emph{density} modes, which RBF facility location targets. This is
the same boundary-point pathology the paper documents for $k$-center
(Table~\ref{tab:cf}) and quality-blind coverage (Table~\ref{tab:quality}): the
correspondence holds where the lemma's condition $m(\kappa(0)-\kappa(\sqrt2))\gg L\eta$
holds, and degrades with it.}
\setlength{\tabcolsep}{5pt}
\begin{tabular}{lccc}
\toprule
$k=C$ & CIFAR-100 & Flowers-102 & Food-101 \\
\midrule
random          & $0.280{\scriptstyle\pm.014}$ / $66.7$ & $0.486{\scriptstyle\pm.028}$ / $60.7$ & $0.179{\scriptstyle\pm.005}$ / $66.3$ \\
$k$-center      & $0.243{\scriptstyle\pm.004}$ / $65.0$ & $0.639{\scriptstyle\pm.015}$ / $76.7$ & $0.195{\scriptstyle\pm.010}$ / $60.3$ \\
greedy-proj (literal OMP) & $0.243{\scriptstyle\pm.004}$ / $64.0$ & $0.459{\scriptstyle\pm.006}$ / $56.7$ & $0.163{\scriptstyle\pm.005}$ / $65.0$ \\
greedy-FL       & $\mathbf{0.403}{\scriptstyle\pm.001}$ / $\mathbf{78.3}$ & $\mathbf{0.683}{\scriptstyle\pm.011}$ / $\mathbf{78.7}$ & $\mathbf{0.254}{\scriptstyle\pm.010}$ / $\mathbf{70.7}$ \\
\bottomrule
\end{tabular}
\label{tab:omp}
\end{table}

\noindent Two readings follow. \emph{For the theory--experiment bridge:} the theorems'
content is the scalar-vs-set gap --- any \emph{fixed} ranking under-covers relative to
\emph{adaptive} coverage --- and both greedy projection and facility location are
adaptive coverage selectors that coincide in the analyzed instance. On real features
the operative modes are density modes of the learner geometry rather than spectral
directions of the raw second moment, so facility location on the RBF Gram is the
faithful instantiation of ``cover the learner's modes'', and we present it as such
rather than as an approximation to literal trace-reduction. \emph{For practice:} the
divergence is itself an instance of the paper's quality axis --- objectives that reward
extremes (trace-reduction, $k$-center) select atypical points; density-weighted
coverage does not --- reinforcing that the \emph{choice} of coverage objective matters
beyond the scalar-vs-set distinction.

\begin{table}[h]\centering\small
\caption{\textbf{Full per-selector test accuracy at $k=C$} (modes $=$ labels; mean $\pm$
se, $5$ seeds), complementing Table~\ref{tab:real}, which reports only the gap to the
\emph{best} selector. Class-\emph{balanced} prototypicality achieves the highest mean accuracy
on every dataset; raw per-point scalars (leverage, prototypicality, LOO, Data Shapley) lose
to their balanced versions and often to random; greedy coverage is strong but is itself
matched (Oxford Pets) or beaten by balanced prototypicality when modes $=$ labels. Data Shapley (a TMC estimate) is
omitted on the $100$-class datasets for cost.}
\setlength{\tabcolsep}{4pt}
\begin{tabular}{lcccc}
\toprule
selector ($k=C$) & CIFAR-10 & Oxford Pets & CIFAR-100 & Flowers-102 \\
\midrule
\multicolumn{5}{l}{\emph{raw scalar / random}}\\
random & $0.554{\scriptstyle\pm.017}$ & $0.453{\scriptstyle\pm.026}$ & $0.246{\scriptstyle\pm.010}$ & $0.479{\scriptstyle\pm.005}$ \\
leverage & $0.373{\scriptstyle\pm.037}$ & $0.413{\scriptstyle\pm.015}$ & $0.193{\scriptstyle\pm.010}$ & $0.364{\scriptstyle\pm.013}$ \\
prototypicality & $0.353{\scriptstyle\pm.021}$ & $0.187{\scriptstyle\pm.010}$ & $0.231{\scriptstyle\pm.007}$ & $0.100{\scriptstyle\pm.007}$ \\
LOO / influence & $0.419{\scriptstyle\pm.060}$ & $0.235{\scriptstyle\pm.029}$ & $0.275{\scriptstyle\pm.005}$ & $0.459{\scriptstyle\pm.006}$ \\
Data Shapley & $0.563{\scriptstyle\pm.032}$ & $0.505{\scriptstyle\pm.017}$ & --- & --- \\
\midrule
\multicolumn{5}{l}{\emph{class-balanced (same scores, equal per-class quota)}}\\
bal-random & $0.672{\scriptstyle\pm.012}$ & $0.606{\scriptstyle\pm.008}$ & $0.348{\scriptstyle\pm.007}$ & $0.762{\scriptstyle\pm.013}$ \\
bal-leverage & $0.620{\scriptstyle\pm.019}$ & $0.592{\scriptstyle\pm.014}$ & $0.276{\scriptstyle\pm.005}$ & $0.746{\scriptstyle\pm.004}$ \\
bal-prototypicality & $\mathbf{0.851}{\scriptstyle\pm.007}$ & $\mathbf{0.675}{\scriptstyle\pm.010}$ & $\mathbf{0.475}{\scriptstyle\pm.005}$ & $\mathbf{0.852}{\scriptstyle\pm.005}$ \\
bal-LOO & $0.579{\scriptstyle\pm.031}$ & $0.626{\scriptstyle\pm.036}$ & $0.365{\scriptstyle\pm.012}$ & $0.761{\scriptstyle\pm.009}$ \\
bal-Shapley & $0.781{\scriptstyle\pm.015}$ & $0.638{\scriptstyle\pm.014}$ & --- & --- \\
\midrule
\multicolumn{5}{l}{\emph{coverage}}\\
$k$-center & $0.527{\scriptstyle\pm.022}$ & $0.552{\scriptstyle\pm.014}$ & $0.249{\scriptstyle\pm.006}$ & $0.649{\scriptstyle\pm.007}$ \\
greedy facility location & $0.759{\scriptstyle\pm.020}$ & $0.671{\scriptstyle\pm.008}$ & $0.394{\scriptstyle\pm.008}$ & $0.681{\scriptstyle\pm.002}$ \\
\bottomrule
\end{tabular}
\label{tab:fullsel}
\end{table}

\section{Additional robustness experiments}\label{app:robust}

This appendix collects further robustness evidence: training-dynamics valuation
baselines, downstream learners, full budget sweeps, hyperparameter sensitivity, two
checks of the theory on real pools, a text modality, end-to-end training (a from-scratch
CNN and a LoRA fine-tune), and encoders. Unless stated otherwise, features are the same
frozen embeddings as \S\ref{sec:exp} (PCA-$64$, pools of $4000$, $3$ seeds, budget
$k=C$).

\paragraph{Training-dynamics valuation baselines: EL2N, GraNd, TracIn-CP.} The main
text instantiates the concentration mechanism with geometry scores (leverage,
prototypicality), LOO, TMC-Shapley, and exact KNN-Shapley. We add the
training-dynamics family used in data pruning --- EL2N and GraNd~\citep{paul2021} and
TracIn-CP~\citep{pruthi2020}, computed for a linear softmax probe (full-batch GD, $5$
inits; TracIn from checkpointed gradients against a held-out validation split of the
pool). Because these are \emph{difficulty} scores, they carry a quality orientation on
top of the coverage axis, so we evaluate the full $2\times2$: \{hard-first,
easy-first\} $\times$ \{raw top-$k$, class-balanced\}
(Table~\ref{tab:valextra}). Three regularities emerge. \emph{(i)} Every \emph{raw}
top-$k$ --- all three scores, both orientations, all three datasets, $18/18$ cases ---
concentrates: it covers fewer classes than random and scores \emph{below} random, the
occupancy mechanism on a third score family. \emph{(ii)} Class-balancing restores full
class coverage ($C/C$) in every case and improves test accuracy over the matched raw
variant in $17$ of $18$ cases (the exception, EL2N-hard on CIFAR-10, stays far below
random either way); whether it also recovers \emph{good} accuracy depends on the within-class
ranking it inherits: EL2N-easy balanced is the best of its family on all three datasets
($0.848/0.497/0.829$), while EL2N-hard
balanced stays below random on both CIFAR variants --- coverage is fixed, quality is
not, the quality axis of \S\ref{sec:limits} (Table~\ref{tab:quality}) on real scores.
\emph{(iii)} Which orientation is right is score- and dataset-dependent (GraNd-hard
beats GraNd-easy on CIFAR-10 and Flowers but not CIFAR-100), so no fixed per-point
recipe substitutes for the structural corrections; what is uniform across all $36$
variants is the two axes themselves: raw rankings concentrate, and balancing repairs
exactly the coverage axis. The practical reading is the decoupled recipe: choose a
quality-sane within-class ranking, then apply the coverage correction; neither alone
suffices for difficulty scores.

\begin{table}[h]\centering\small
\caption{\textbf{Training-dynamics scores obey the same law} ($k=C$, $3$ seeds,
test accuracy; pool excludes the $1000$-point validation split, so the
reference rows differ slightly from Table~\ref{tab:fullsel}. Classes covered: every
raw variant covers fewer classes than random --- as few as $2.0/10$, $20.3/100$,
$5.0/102$ --- while every balanced variant covers all $C$.) Raw top-$k$ of every
score and orientation trails random ($18/18$); balancing restores coverage in all
cases and good accuracy exactly when the within-class ranking is quality-sane
(EL2N-easy best of family on all three datasets).}
\setlength{\tabcolsep}{4pt}
\begin{tabular}{lccc}
\toprule
selector ($k=C$) & CIFAR-10 & CIFAR-100 & Flowers-102 \\
\midrule
random & $0.491{\scriptstyle\pm.030}$ & $0.275{\scriptstyle\pm.005}$ & $0.474{\scriptstyle\pm.023}$ \\
greedy facility location & $0.759{\scriptstyle\pm.023}$ & $0.395{\scriptstyle\pm.004}$ & $0.688{\scriptstyle\pm.013}$ \\
\midrule
EL2N hard, raw / balanced & $0.187$ / $0.173$ & $0.060$ / $0.099$ & $0.178$ / $0.502$ \\
EL2N easy, raw / balanced & $0.383$ / $\mathbf{0.848}$ & $0.209$ / $\mathbf{0.497}$ & $0.050$ / $\mathbf{0.829}$ \\
GraNd hard, raw / balanced & $0.378$ / $0.661$ & $0.168$ / $0.287$ & $0.236$ / $0.789$ \\
GraNd easy, raw / balanced & $0.278$ / $0.524$ & $0.227$ / $0.365$ & $0.234$ / $0.696$ \\
TracIn hard, raw / balanced & $0.228$ / $0.284$ & $0.097$ / $0.265$ & $0.150$ / $0.695$ \\
TracIn easy, raw / balanced & $0.212$ / $0.296$ & $0.060$ / $0.164$ & $0.080$ / $0.660$ \\
\bottomrule
\end{tabular}
\label{tab:valextra}
\end{table}

\paragraph{More downstream learners.} Table~\ref{tab:learners} repeats the label-free
selection comparison of Table~\ref{tab:unlab} with logistic regression and $1$-NN in
place of the ridge probe. The coverage gap is stable across learners ($+0.12$ to
$+0.14$ on CIFAR-100, $+0.20$ to $+0.23$ on Flowers-102) and the raw leverage top-$k$
stays below random for every learner: together with the SGD-MLP result
(Table~\ref{tab:mlp}), the selection interaction is a property of the selected
\emph{data}, not of the estimator trained on it.

\begin{table}[h]\centering\small
\caption{\textbf{The gap is learner-independent} (label-free selection, $k=C$, $3$
seeds): the coverage gap (greedy-FL $-$ random) and the below-random raw leverage
top-$k$ persist across ridge, logistic regression, and $1$-NN. Gaps are computed from unrounded means, so they can differ by $0.001$ from differences of the displayed values.}
\setlength{\tabcolsep}{4pt}
\begin{tabular}{llccccc}
\toprule
dataset & learner & random & leverage & $k$-center & greedy-FL & gap (FL$-$rand) \\
\midrule
CIFAR-100 & ridge & $0.280{\scriptstyle\pm.014}$ & $0.179{\scriptstyle\pm.005}$ & $0.243{\scriptstyle\pm.004}$ & $\mathbf{0.403}{\scriptstyle\pm.001}$ & $+0.123$ \\
 & logistic & $0.291{\scriptstyle\pm.016}$ & $0.212{\scriptstyle\pm.045}$ & $0.175{\scriptstyle\pm.007}$ & $\mathbf{0.412}{\scriptstyle\pm.005}$ & $+0.121$ \\
 & $1$-NN & $0.263{\scriptstyle\pm.013}$ & $0.193{\scriptstyle\pm.012}$ & $0.137{\scriptstyle\pm.004}$ & $\mathbf{0.399}{\scriptstyle\pm.002}$ & $+0.136$ \\
\midrule
Flowers-102 & ridge & $0.486{\scriptstyle\pm.028}$ & $0.374{\scriptstyle\pm.011}$ & $0.639{\scriptstyle\pm.015}$ & $\mathbf{0.683}{\scriptstyle\pm.011}$ & $+0.197$ \\
 & logistic & $0.481{\scriptstyle\pm.026}$ & $0.408{\scriptstyle\pm.004}$ & $0.509{\scriptstyle\pm.028}$ & $\mathbf{0.702}{\scriptstyle\pm.002}$ & $+0.220$ \\
 & $1$-NN & $0.455{\scriptstyle\pm.018}$ & $0.377{\scriptstyle\pm.011}$ & $0.445{\scriptstyle\pm.007}$ & $\mathbf{0.682}{\scriptstyle\pm.004}$ & $+0.227$ \\
\bottomrule
\end{tabular}
\label{tab:learners}
\end{table}

\paragraph{Full budget sweeps.} The main text emphasizes the binding budget $k=C$
because the theory locates the largest gap there; Table~\ref{tab:budgets} reports the
complete picture at $k=C,2C,4C$ for both regimes, from the same runs. Two facts
emerge, both predicted. \emph{Label-free:} greedy coverage beats random at every
budget on every dataset, with the margin shrinking as the budget loosens
(CIFAR-100: $+0.131\to+0.098\to+0.061$), and the raw scalar top-$k$s remain at or
below random throughout --- the $\delta$-shape on real data. \emph{Modes $=$ labels:}
the FL-over-best-raw-scalar gap is $\delta$-shaped (e.g.\ Flowers
$+0.223\to+0.143\to+0.086$; CIFAR-10 even reverses sign by $4C$, $-0.009$), while
best-\emph{balanced} matches or beats FL at \emph{every} budget on every dataset (every
entry is negative, all but Oxford Pets at $k=C$ by more than one standard error) ---
balancing is not a binding-budget artifact.

\begin{table}[h]\centering\footnotesize
\caption{\textbf{Complete budget sweeps} ($k=C,2C,4C$; mean $\pm$ se over seeds; same
runs as the main-text tables). \emph{Top:} label-free selection, test accuracy.
\emph{Bottom:} modes $=$ labels, the two headline gaps. Gaps are computed from unrounded means, so they can differ by $0.001$ from differences of the displayed values.}
\setlength{\tabcolsep}{4pt}
\begin{tabular}{llccccc}
\toprule
\multicolumn{7}{l}{\emph{(a) label-free selection: accuracy}}\\
dataset & $k$ & random & leverage & typicality & $k$-center & greedy-FL \\
\midrule
CIFAR-100 & $C$ & $0.274{\scriptstyle\pm.009}$ & $0.175{\scriptstyle\pm.007}$ & $0.224{\scriptstyle\pm.006}$ & $0.247{\scriptstyle\pm.004}$ & $\mathbf{0.405}{\scriptstyle\pm.004}$ \\
 & $2C$ & $0.387{\scriptstyle\pm.013}$ & $0.303{\scriptstyle\pm.010}$ & $0.315{\scriptstyle\pm.010}$ & $0.355{\scriptstyle\pm.007}$ & $\mathbf{0.485}{\scriptstyle\pm.002}$ \\
 & $4C$ & $0.461{\scriptstyle\pm.012}$ & $0.385{\scriptstyle\pm.006}$ & $0.385{\scriptstyle\pm.008}$ & $0.432{\scriptstyle\pm.005}$ & $\mathbf{0.523}{\scriptstyle\pm.004}$ \\
\midrule
Flowers-102 & $C$ & $0.490{\scriptstyle\pm.016}$ & $0.375{\scriptstyle\pm.006}$ & $0.224{\scriptstyle\pm.014}$ & $0.642{\scriptstyle\pm.011}$ & $\mathbf{0.686}{\scriptstyle\pm.006}$ \\
 & $2C$ & $0.624{\scriptstyle\pm.011}$ & $0.500{\scriptstyle\pm.011}$ & $0.382{\scriptstyle\pm.011}$ & $0.698{\scriptstyle\pm.005}$ & $\mathbf{0.759}{\scriptstyle\pm.009}$ \\
 & $4C$ & $0.678{\scriptstyle\pm.008}$ & $0.588{\scriptstyle\pm.013}$ & $0.487{\scriptstyle\pm.012}$ & $0.721{\scriptstyle\pm.010}$ & $\mathbf{0.770}{\scriptstyle\pm.007}$ \\
\midrule
Food-101 & $C$ & $0.175{\scriptstyle\pm.004}$ & $0.129{\scriptstyle\pm.005}$ & $0.063{\scriptstyle\pm.002}$ & $0.204{\scriptstyle\pm.008}$ & $\mathbf{0.254}{\scriptstyle\pm.006}$ \\
 & $2C$ & $0.279{\scriptstyle\pm.010}$ & $0.224{\scriptstyle\pm.005}$ & $0.128{\scriptstyle\pm.001}$ & $0.281{\scriptstyle\pm.009}$ & $\mathbf{0.333}{\scriptstyle\pm.003}$ \\
 & $4C$ & $0.345{\scriptstyle\pm.011}$ & $0.297{\scriptstyle\pm.009}$ & $0.196{\scriptstyle\pm.005}$ & $0.341{\scriptstyle\pm.007}$ & $\mathbf{0.370}{\scriptstyle\pm.006}$ \\
\bottomrule
\end{tabular}

\vspace{5pt}
\begin{tabular}{lcccccc}
\toprule
\multicolumn{7}{l}{\emph{(b) modes $=$ labels: gaps across budgets}}\\
 & \multicolumn{3}{c}{greedy $-$ best \emph{raw} scalar} & \multicolumn{3}{c}{greedy $-$ best \emph{balanced} scalar} \\
dataset & $C$ & $2C$ & $4C$ & $C$ & $2C$ & $4C$ \\
\midrule
CIFAR-10 & $+0.196{\scriptstyle\pm.042}$ & $+0.072{\scriptstyle\pm.009}$ & $-0.009{\scriptstyle\pm.010}$ & $-0.092{\scriptstyle\pm.020}$ & $-0.018{\scriptstyle\pm.009}$ & $-0.027{\scriptstyle\pm.006}$ \\
Oxford Pets & $+0.166{\scriptstyle\pm.021}$ & $+0.042{\scriptstyle\pm.015}$ & $+0.018{\scriptstyle\pm.011}$ & $-0.004{\scriptstyle\pm.015}$ & $-0.021{\scriptstyle\pm.014}$ & $-0.027{\scriptstyle\pm.009}$ \\
CIFAR-100 & $+0.119{\scriptstyle\pm.012}$ & $+0.067{\scriptstyle\pm.014}$ & $+0.040{\scriptstyle\pm.007}$ & $-0.081{\scriptstyle\pm.012}$ & $-0.092{\scriptstyle\pm.008}$ & $-0.087{\scriptstyle\pm.003}$ \\
Flowers-102 & $+0.223{\scriptstyle\pm.006}$ & $+0.143{\scriptstyle\pm.021}$ & $+0.086{\scriptstyle\pm.004}$ & $-0.171{\scriptstyle\pm.006}$ & $-0.124{\scriptstyle\pm.003}$ & $-0.132{\scriptstyle\pm.008}$ \\
Food-101 & $+0.077{\scriptstyle\pm.016}$ & $+0.033{\scriptstyle\pm.014}$ & $+0.016{\scriptstyle\pm.014}$ & $-0.046{\scriptstyle\pm.001}$ & $-0.083{\scriptstyle\pm.002}$ & $-0.084{\scriptstyle\pm.011}$ \\
\bottomrule
\end{tabular}
\label{tab:budgets}
\end{table}

\paragraph{Hyperparameter sensitivity: kernel bandwidth and PCA dimension.} The paper
uses the median-heuristic bandwidth and PCA-$64$ everywhere without tuning;
Table~\ref{tab:sens} varies each around the default. The coverage gap is essentially
flat across a $4\times$ range of squared bandwidth ($+0.109$ to $+0.127$ on CIFAR-100;
$+0.194$ to $+0.210$ on Flowers-102) and across a doubling of the PCA dimension; at
$d=32$ the whole representation degrades (random drops too) and the gap attenuates but
remains clearly positive ($+0.062$, $+0.129$). No selector hyperparameter was tuned
anywhere in the paper.

\begin{table}[h]\centering\small
\caption{\textbf{Sensitivity of the coverage gap} (greedy-FL vs random, $k=C$, $3$
seeds) to the RBF bandwidth multiplier $b$ ($\sigma^2=b\times$ median) and the PCA
dimension $d$. Default $b=1$, $d=64$ in bold.}
\setlength{\tabcolsep}{5pt}
\begin{tabular}{lccccc}
\toprule
 & $b{=}0.5$ & $\mathbf{b{=}1}$ & $b{=}2$ & $d{=}32$ & $d{=}128$ \\
\midrule
CIFAR-100 gap  & $+0.127$ & $\mathbf{+0.123}$ & $+0.109$ & $+0.062$ & $+0.118$ \\
Flowers-102 gap & $+0.210$ & $\mathbf{+0.197}$ & $+0.194$ & $+0.129$ & $+0.235$ \\
\bottomrule
\end{tabular}
\label{tab:sens}
\end{table}

\paragraph{The idealization, dialed on real features.} The latent-mode model is the
hard limit of a dial that real data sits partway along, and the dial can be turned
\emph{on the real features themselves}. Starting from the frozen ViT pool, we either
shrink within-class variance toward the class means (interpolation $\lambda$, toward
the hard block limit) or add isotropic noise (toward the diffuse limit), manipulating
only the selection/training geometry --- test features are untouched and no test
labels are used. Table~\ref{tab:blockness}: the coverage gap at $k=C$ generally
increases with the measured block contrast through the real datasets' operating point
and vanishes (or reverses) at the diffuse end --- monotonically on Flowers-102, where
the sweep runs $-0.051\to+0.219$ (real) $\to+0.460$ without a single inversion, and
with one small local reversal on CIFAR-100 ($+0.152\to+0.148$ between $\lambda=0.3$
and $0.6$). The mechanism the
hard model isolates is present in real geometry, with the model as its endpoint; this
is the real-feature analogue of the synthetic causal law (Table~\ref{tab:C}).

\begin{table}[h]\centering\small
\caption{\textbf{The causal law on real features} ($k=C$, $3$ seeds). The pool
geometry is dialed from diffuse (noise $\sigma\times$ feature scale) through the
\emph{real} features to near-block (within-class shrinkage $\lambda$); block contrast
is the measured within$-$cross class kernel gap, and the selection gap
(greedy-FL $-$ random) generally increases with it (monotonically on Flowers-102, one
small local reversal on CIFAR-100), vanishing at the diffuse end ---
effective rank spans $58\!\to\!15$ (CIFAR-100). No test-side information is touched.}
\setlength{\tabcolsep}{5pt}
\begin{tabular}{lcccc}
\toprule
 & \multicolumn{2}{c}{CIFAR-100} & \multicolumn{2}{c}{Flowers-102} \\
geometry & block contrast & gap (FL$-$rand) & block contrast & gap (FL$-$rand) \\
\midrule
noise $\sigma{=}2.0$ & $0.037$ & $-0.007$ & $0.056$ & $-0.051$ \\
noise $\sigma{=}1.0$ & $0.094$ & $+0.044$ & $0.146$ & $+0.093$ \\
noise $\sigma{=}0.5$ & $0.155$ & $+0.110$ & $0.243$ & $+0.187$ \\
\textbf{real features} & $0.197$ & $+0.135$ & $0.312$ & $+0.219$ \\
shrink $\lambda{=}0.3$ & $0.299$ & $+0.152$ & $0.422$ & $+0.267$ \\
shrink $\lambda{=}0.6$ & $0.455$ & $+0.148$ & $0.541$ & $+0.345$ \\
shrink $\lambda{=}0.9$ & $0.598$ & $+0.243$ & $0.618$ & $+0.460$ \\
\bottomrule
\end{tabular}
\label{tab:blockness}
\end{table}

\paragraph{The occupancy formula as a zero-parameter prediction.} Theorem~\ref{thm:gap}'s
coverage count is exact combinatorics, so it can be confronted with real pools with
\emph{no} fitted quantity: with the actual per-class counts $\{n_c\}$ of each pool,
a uniform random $k$-subset covers
$\sum_c\big(1-\binom{N-n_c}{k}/\binom{N}{k}\big)$ classes in expectation.
Table~\ref{tab:occupancy}: observed random coverage matches the prediction to at most
$0.24$ classes (of up to $102$; $6.51$ vs $6.51$ on CIFAR-10) --- the theorem's
combinatorial core is literally measurable on real data. The formula then serves as
the yardstick for concentration: every raw scalar top-$k$ falls \emph{below} it (by up
to $49$ classes for prototypicality on Flowers), while greedy coverage sits far above
--- the theorem's two sides, quantified in classes on real pools.

\begin{table}[h]\centering\small
\caption{\textbf{The occupancy formula predicts real coverage with zero parameters}
(classes covered at $k=C$, $3$ seeds, $300$ random draws per seed). Random selection
sits on the exact prediction $m_\star$ computed from each pool's actual class counts;
every raw scalar top-$k$ falls below it (concentration); greedy coverage exceeds it.}
\setlength{\tabcolsep}{5pt}
\begin{tabular}{lccccccc}
\toprule
dataset & $C$ & $m_\star$ (exact) & random (obs.) & leverage & proto & typicality & greedy-FL \\
\midrule
CIFAR-10    & $10$  & $6.51$  & $6.51$  & $5.7$  & $2.7$  & $3.7$  & $8.7$ \\
Oxford Pets & $37$  & $23.64$ & $23.56$ & $21.0$ & $7.0$  & $11.0$ & $32.0$ \\
CIFAR-100   & $100$ & $63.44$ & $63.50$ & $56.7$ & $31.3$ & $52.7$ & $78.3$ \\
Flowers-102 & $102$ & $59.28$ & $59.04$ & $49.0$ & $10.3$ & $38.7$ & $78.7$ \\
Food-101    & $101$ & $64.02$ & $63.93$ & $53.0$ & $13.3$ & $51.0$ & $70.7$ \\
\bottomrule
\end{tabular}
\label{tab:occupancy}
\end{table}

\paragraph{Text modality.} All real-data experiments above are vision. Repeating the
core comparison on 20~Newsgroups with frozen MiniLM sentence embeddings
(mean-pooled hidden states, PCA-$64$, ridge probe, $3$ seeds) reproduces the full
pattern on language (Table~\ref{tab:text}): at the binding budget $k=C=20$, every raw
scalar top-$k$ covers fewer classes than random and falls well below it (leverage
$0.156$, global typicality $0.112$ vs random $0.275$), greedy coverage beats random
($0.358$), and class-balanced prototypicality is best ($0.426$, covering $20/20$
classes); the ordering is unchanged at $2C$ and $4C$. The mechanism is not a property
of image encoders.

\begin{table}[h]\centering\small
\caption{\textbf{The gap on language embeddings} (20 Newsgroups, frozen MiniLM,
$C=20$, $3$ seeds; accuracy / classes covered). The vision pattern replicates:
raw scalar top-$k$ below random with under-coverage, coverage above random,
balanced prototypicality best.}
\setlength{\tabcolsep}{5pt}
\begin{tabular}{lccc}
\toprule
selector & $k=C$ & $k=2C$ & $k=4C$ \\
\midrule
random & $0.275{\scriptstyle\pm.024}$ / $13.0$ & $0.298{\scriptstyle\pm.009}$ / $17.3$ & $0.339{\scriptstyle\pm.013}$ / $20.0$ \\
leverage top-$k$ & $0.156{\scriptstyle\pm.002}$ / $12.0$ & $0.179{\scriptstyle\pm.014}$ / $16.0$ & $0.208{\scriptstyle\pm.009}$ / $19.3$ \\
typicality top-$k$ & $0.112{\scriptstyle\pm.007}$ / $10.7$ & $0.134{\scriptstyle\pm.019}$ / $14.3$ & $0.166{\scriptstyle\pm.017}$ / $18.0$ \\
$k$-center & $0.146{\scriptstyle\pm.013}$ / $11.3$ & $0.192{\scriptstyle\pm.014}$ / $16.0$ & $0.225{\scriptstyle\pm.008}$ / $19.7$ \\
greedy-FL & $0.358{\scriptstyle\pm.016}$ / $14.7$ & $0.383{\scriptstyle\pm.016}$ / $18.7$ & $0.414{\scriptstyle\pm.009}$ / $20.0$ \\
bal-prototypicality & $\mathbf{0.426}{\scriptstyle\pm.009}$ / $20.0$ & $\mathbf{0.439}{\scriptstyle\pm.004}$ / $20.0$ & $\mathbf{0.464}{\scriptstyle\pm.002}$ / $20.0$ \\
\bottomrule
\end{tabular}
\label{tab:text}
\end{table}

\paragraph{End-to-end pilot: a CNN trained from scratch on the selected subset.} As a
first step past frozen representations, we train a small CNN \emph{from raw
$32^2$ pixels} on the selected CIFAR-100 subsets (standard augmentation, Adam, $60$
epochs; selection still in the frozen ViT geometry; $3$ seeds). Absolute accuracies
are necessarily tiny (a from-scratch network on $100$--$400$ images), so we report
this as a pilot; but the selector ordering matches the frozen-feature experiments at
\emph{every} budget (Table~\ref{tab:scratch}): leverage top-$k$ $<$ $k$-center $<$
random $<$ greedy-FL $\le$ balanced prototypicality, with the raw top-$k$ clearly
below random ($0.064$ vs $0.094$ at $k=4C$) and coverage/balance clearly above
($0.109$/$0.120$). End-to-end training does not wash out the selection effect; the LoRA
experiment below strengthens this further.

\begin{table}[h]\centering\small
\caption{\textbf{End-to-end pilot} (CNN trained from scratch on raw CIFAR-100 pixels
on the selected subset; $3$ seeds). Absolute accuracies are small at these subset
sizes, but the selector ordering of the frozen-feature experiments persists at every
budget.}
\setlength{\tabcolsep}{5pt}
\begin{tabular}{lccccc}
\toprule
$k$ & random & leverage & $k$-center & greedy-FL & bal-proto \\
\midrule
$C$  & $0.053{\scriptstyle\pm.003}$ & $0.041{\scriptstyle\pm.004}$ & $0.044{\scriptstyle\pm.000}$ & $\mathbf{0.066}{\scriptstyle\pm.008}$ & $\mathbf{0.066}{\scriptstyle\pm.004}$ \\
$2C$ & $0.062{\scriptstyle\pm.004}$ & $0.044{\scriptstyle\pm.006}$ & $0.054{\scriptstyle\pm.006}$ & $0.072{\scriptstyle\pm.002}$ & $\mathbf{0.082}{\scriptstyle\pm.003}$ \\
$4C$ & $0.094{\scriptstyle\pm.003}$ & $0.064{\scriptstyle\pm.004}$ & $0.066{\scriptstyle\pm.005}$ & $0.109{\scriptstyle\pm.006}$ & $\mathbf{0.120}{\scriptstyle\pm.004}$ \\
\bottomrule
\end{tabular}
\label{tab:scratch}
\end{table}

\paragraph{LoRA fine-tuning on the selected subset.} Going further than the pilot:
we inject rank-$8$ LoRA adapters into every attention qkv/proj layer of the frozen
ViT-B/16, add a fresh head, and train adapters$+$head on the subset chosen by each
selector (selection still in the frozen PCA-$64$ geometry, as everywhere in the
paper); CIFAR-100, $k\in\{C,4C\}$, $3$ seeds (Table~\ref{tab:lora}). The frozen-feature
ordering replicates in all $24$ runs: at $k=C$, leverage top-$k$ $0.201$ $<$ random $0.253$ $<$ greedy-FL
$0.380$ $<$ balanced prototypicality $0.488$; at $k=4C$ the coverage gap shrinks
($+0.075$) and balancing stays best --- the $\delta$-shape under parameter-efficient
fine-tuning of the very backbone whose geometry drives selection.

\begin{table}[h]\centering\small
\caption{\textbf{LoRA fine-tuning of the selection backbone} (rank-$8$ adapters on all
ViT-B/16 attention layers $+$ fresh head, trained on the selected subset; CIFAR-100,
$3$ seeds, mean $\pm$ se). The frozen-feature ordering --- leverage $<$ random $<$
greedy-FL $\le$ balanced prototypicality --- holds in every run.}
\setlength{\tabcolsep}{5pt}
\begin{tabular}{lcccc}
\toprule
$k$ & random & leverage & greedy-FL & bal-proto \\
\midrule
$C$  & $0.253{\scriptstyle\pm.014}$ & $0.201{\scriptstyle\pm.011}$ & $0.380{\scriptstyle\pm.001}$ & $\mathbf{0.488}{\scriptstyle\pm.007}$ \\
$4C$ & $0.539{\scriptstyle\pm.001}$ & $0.439{\scriptstyle\pm.007}$ & $0.614{\scriptstyle\pm.002}$ & $\mathbf{0.648}{\scriptstyle\pm.004}$ \\
\bottomrule
\end{tabular}
\label{tab:lora}
\end{table}

\paragraph{Representation sensitivity: more encoders.} Table~\ref{tab:encoders}
repeats the two core comparisons --- the label-free coverage gap and the
balanced-prototypicality-vs-coverage result --- on ImageNet ResNet-50 and CLIP
ViT-B/32 in addition to the paper's ViT-B/16. Both core results replicate on every
backbone: greedy coverage beats random in all six backbone$\times$dataset cells
($+0.06$ to $+0.20$) and class-balanced prototypicality is the best selector in all
six. The \emph{severity} of the raw top-$k$ pathology is representation-dependent, as
the theory says it must be: leverage falls below random in five of six cells but
\emph{ties} random on CLIP/CIFAR-100 ($+0.011$) --- CLIP's contrastively-trained
geometry spreads ridge leverage more uniformly, attenuating (not reversing) the
concentration failure. The gap is a property of the learner geometry's redundancy
structure, and it travels with the encoder exactly as condition (iii) of
Corollary~\ref{cor:phase}(b) predicts.

\begin{table}[h]\centering\small
\caption{\textbf{The mechanism is encoder-robust} ($k=C$, $3$ seeds, accuracy;
label-free selectors plus class-balanced prototypicality). Across ImageNet ViT-B/16,
ImageNet ResNet-50, and CLIP ViT-B/32: greedy coverage beats random and balanced
prototypicality is best in every cell; the raw leverage top-$k$ trails random in five
of six cells (ties on CLIP/CIFAR-100).}
\setlength{\tabcolsep}{4pt}
\begin{tabular}{llccccc}
\toprule
backbone & dataset & random & leverage & $k$-center & greedy-FL & bal-proto \\
\midrule
ViT-B/16 & CIFAR-100 & $0.280{\scriptstyle\pm.014}$ & $0.179{\scriptstyle\pm.005}$ & $0.243{\scriptstyle\pm.004}$ & $0.403{\scriptstyle\pm.001}$ & $\mathbf{0.482}{\scriptstyle\pm.012}$ \\
 & Flowers-102 & $0.486{\scriptstyle\pm.028}$ & $0.374{\scriptstyle\pm.011}$ & $0.639{\scriptstyle\pm.015}$ & $0.683{\scriptstyle\pm.011}$ & $\mathbf{0.843}{\scriptstyle\pm.005}$ \\
\midrule
ResNet-50 & CIFAR-100 & $0.182{\scriptstyle\pm.009}$ & $0.127{\scriptstyle\pm.003}$ & $0.209{\scriptstyle\pm.004}$ & $0.240{\scriptstyle\pm.012}$ & $\mathbf{0.307}{\scriptstyle\pm.004}$ \\
 & Flowers-102 & $0.269{\scriptstyle\pm.009}$ & $0.188{\scriptstyle\pm.004}$ & $0.310{\scriptstyle\pm.005}$ & $0.377{\scriptstyle\pm.005}$ & $\mathbf{0.492}{\scriptstyle\pm.004}$ \\
\midrule
CLIP ViT-B/32 & CIFAR-100 & $0.211{\scriptstyle\pm.007}$ & $0.222{\scriptstyle\pm.009}$ & $0.300{\scriptstyle\pm.004}$ & $0.289{\scriptstyle\pm.005}$ & $\mathbf{0.348}{\scriptstyle\pm.007}$ \\
 & Flowers-102 & $0.303{\scriptstyle\pm.008}$ & $0.213{\scriptstyle\pm.005}$ & $0.371{\scriptstyle\pm.014}$ & $0.397{\scriptstyle\pm.006}$ & $\mathbf{0.474}{\scriptstyle\pm.004}$ \\
\bottomrule
\end{tabular}
\label{tab:encoders}
\end{table}

\section{Numerical verification}\label{app:num}

All analytic claims are reproduced by the self-contained synthetic suite
(\texttt{experiments.py}, numpy, ${\sim}1.5$s):
(i) the marginal $g_i(S)$ over random contexts has spread/$|$mean$|\approx1.5$ and is
bimodal ($40\%$ large, $57\%\approx0$), matching Lemma~\ref{lem:bimodal};
(ii) leave-one-out is $\equiv0$ to machine precision and Data Shapley equals $\rho/m$ to
three digits, matching Appendix~\ref{app:scalars};
(iii) at the binding budget $k=C$ greedy coverage scores $0.895$ versus best per-point
top-$k$ $0.637$ (gap $+0.258$), and \emph{estimated} Data Shapley top-$k$ dips below
random near $k=C$, matching Theorem~\ref{thm:gap} and the concentration corollary;
(iv) the soft-mode marginals are $0.117\approx\pi_j$ (uncovered) and $0.004=O(\eta^2)$
(covered), matching Appendix~\ref{app:soft};
(v) sweeping embedding noise, the $k=C$ gap falls with the off-diagonal block mass
and collapses from $+0.35$ to within $0.025$ of zero as the effective rank rises from $15.5$ to $29.8$
(Table~\ref{tab:C}), the controlled causal law.

\paragraph{Direct validation of the occupancy theorem.} Independently of the ML suite,
we check the latent-mode combinatorics directly (\texttt{verify\_theory.py},
$2\times10^4$ trials, Table~\ref{tab:theorycheck}): for $r=10$ equal modes of size
$m=5$, a uniform $k$-subset --- the top-$k$ of any symmetric valuation
(Theorem~\ref{thm:gap}) --- agrees with the closed-form $m_\star(k)$ within $0.007$ modes
(the maximum error across $k$); a concentrating mode-constant score stays at or below $m_\star(k)$
(Proposition~\ref{prop:general}); greedy covers $\min(k,r)$; and with \emph{unequal}
sizes (nine singletons and one mode of $91$ points) a rare-favoring score covers $9.0$
modes versus the uniform $1.8$, exactly the reversal of Remark~\ref{rem:unequal}.

\paragraph{Boundary cases of the corrected statement.} The same script verifies every
boundary case of Theorem~\ref{thm:gap} and Corollary~\ref{cor:phase}(a) exactly:
$\delta_N(1)=0$ to machine precision (a single draw always covers a fresh mode);
$\delta_N(k)>0$ strictly on the whole range $2\le k\le N-m$; $\delta_N(k)=0$ identically
for $k>N-m$ (every $k$-subset covers all modes); $\delta_N\equiv0$ when $m=1$ (no
redundancy, no gap); and, in the $N\to\infty$ limit form, the union bound
$\delta(k)\le\binom{k}{2}/r$ for $k\le r$ and the exact loose-budget decay
$\delta(k)=r(1-1/r)^k$ for $k\ge r$, with $\delta(r)=3.49$ at $r=10$, within $O(1)$ of $r/e=3.68$.

\begin{table}[h]\centering\footnotesize
\caption{Direct validation of the occupancy theorem ($r=10$ modes, equal size $m=5$
unless noted; $2\times10^4$ trials). Uniform sampling agrees with the closed form
$m_\star(k)$ within $0.007$ modes (Thm~\ref{thm:gap}); a concentrating score stays
$\le m_\star$ (Prop~\ref{prop:general}); unequal modes reverse the bound
(Rem~\ref{rem:unequal}).}
\begin{tabular}{rcccc}
\toprule
$k$ & $m_\star(k)$ (theory) & uniform (sim) & concentrating (sim) & greedy \\
\midrule
$2$  & $1.918$ & $1.918$ & $1.000$ & $2$ \\
$5$  & $4.234$ & $4.235$ & $1.000$ & $5$ \\
$8$  & $5.985$ & $5.984$ & $2.000$ & $8$ \\
$10$ & $6.894$ & $6.899$ & $2.000$ & $10$ \\
\midrule
\multicolumn{5}{l}{\emph{unequal} sizes $9{\times}1{+}91$, $k=9$:\quad uniform $1.807$,\quad rare-favoring $9.000$ (Rem~\ref{rem:unequal})}\\
\bottomrule
\end{tabular}
\label{tab:theorycheck}
\end{table}

\FloatBarrier
\section{Information-theoretic selection and mode-aligned targets}\label{app:infotheory}

\subsection{Coverage as Bayesian experimental design}
Consider the Gaussian linear model $f(x)=\theta^\top\varphi(x)$ with prior
$\theta\sim\mathcal N(0,I)$ and observations $y_i=f(x_i)+\epsilon_i$,
$\epsilon_i\sim\mathcal N(0,\sigma^2)$. After observing $S$, the posterior covariance of
$\theta$ is $(I+\sigma^{-2}\Phi_S^\top\Phi_S)^{-1}$, which tends to $I-P_S$ as
$\sigma^2\to0$: noiseless observations reveal the component of $\theta$ in
$\mathrm{span}\{\varphi_i:i\in S\}$ exactly and leave the orthogonal complement at its
prior. The posterior predictive variance at $x$ is then
$v_x(S)=\|(I-P_S)\varphi(x)\|^2$, and averaging over the target distribution gives
$\E_x\,v_x(S)=\E_x\,\varphi(x)^\top(I-P_S)\varphi(x)=\tr((I-P_S)\Sigma)=\R(S)$. The
coverage objective is therefore the integrated posterior predictive variance, the
V-optimality criterion of Bayesian design~\citep{chaloner1995}, and the context function is
the expected variance reduction from adding $i$: with the noiseless rank-one update
$v_x(S\cup\{i\})=v_x(S)-c_S(x,i)^2/v_i(S)$, where
$c_S(x,i)=\langle\varphi(x)^{\perp(S)},\varphi_i^{\perp(S)}\rangle$,
\[
\R(S)-\R(S\cup\{i\})=\frac{\E_x\,c_S(x,i)^2}{v_i(S)}
=\frac{\langle\varphi_i^{\perp(S)},\Sigma\,\varphi_i^{\perp(S)}\rangle}{\|\varphi_i^{\perp(S)}\|^2}
=g_i(S),
\]
which is Eq.~\eqref{eq:gi}. Two information-theoretic objectives bracket it. Expected
predictive information gain (EPIG; \citealp{bickfordsmith2023}) replaces the variance
reduction at a target input by the mutual information between the candidate's label and
the target's label, $\E_x\big[-\tfrac12\log(1-\rho_S(x,i)^2)\big]$ with $\rho_S$ the
posterior predictive correlation: the log-loss analogue of $g_i(S)$, averaged over the
target distribution in the same way. Target-agnostic information gain about $f$,
$\tfrac12\log\det(I+\sigma^{-2}K_{SS})$~\citep{krause2008,srinivas2010}, has greedy
marginal $\tfrac12\log(1+\sigma^{-2}v_i(S))$: it scores a candidate by its own posterior
variance, with no $\Sigma$-weighting.

\paragraph{The objectives as selectors.} Because the Gaussian posterior covariance does not
depend on the labels, all three objectives run label-free, exactly like facility location.
We use the protocol of Table~\ref{tab:omp}, with facility location's RBF median-bandwidth kernel as the GP prior
covariance, the pool as the sample of target inputs, and noise $\sigma^2=0.1$ fixed in
advance to the probe's ridge $\lambda$. \emph{Greedy} variants re-condition the posterior
after every selection (rank-one updates); \emph{fixed top-$k$} computes the EPIG score once
at $S=\emptyset$ and uses it as a per-point ranking. (At $S=\emptyset$, BALD is constant under an RBF kernel, i.e.\ random.) In Table~\ref{tab:infotheory} the paper's dichotomy
reappears inside the information-theoretic family: greedy EPIG --- adaptive and
target-aware --- matches facility location on all three datasets. The same EPIG score used
as a fixed ranking concentrates and falls far \emph{below} random, covering fewer classes
than random on every dataset; it behaves like global typicality in
Table~\ref{tab:unlab}, which is what EPIG at $S=\emptyset$ measures (average predictive
correlation with the pool). Target-agnostic information gain lands between random and
facility location, at random on CIFAR-100. With $\sigma^2=0.01$ every conclusion is
unchanged; with $\sigma^2=1$ each selection conditions the posterior less and greedy EPIG
weakens toward its fixed ranking ($0.360$ / $0.579$ / $0.227$ on CIFAR-100 / Flowers-102 / Food-101), still
above random.

\begin{table}[!htb]
\centering
\caption{\textbf{Information-theoretic objectives as label-free selectors} (protocol of
Table~\ref{tab:omp}: $k=C$, three seeds; accuracy $\pm$ se, classes covered in parentheses;
$\sigma^2=0.1$). Adaptive, target-aware greedy EPIG matches facility location; the same
score used as a fixed ranking falls below random.}
\footnotesize\setlength{\tabcolsep}{4pt}
\begin{tabular}{l ccc}
\toprule
selector & CIFAR-100 & Flowers-102 & Food-101 \\
\midrule
random & $0.280{\scriptstyle\pm.014}$ (66.7) & $0.486{\scriptstyle\pm.028}$ (60.7) & $0.179{\scriptstyle\pm.005}$ (66.3) \\
greedy facility location & $0.403{\scriptstyle\pm.001}$ (78.3) & $0.683{\scriptstyle\pm.011}$ (78.7) & $0.254{\scriptstyle\pm.010}$ (70.7) \\
\midrule
information gain, greedy (target-agnostic) & $0.282{\scriptstyle\pm.007}$ (66.3) & $0.643{\scriptstyle\pm.007}$ (73.7) & $0.229{\scriptstyle\pm.004}$ (63.3) \\
EPIG, greedy (target-aware) & $0.411{\scriptstyle\pm.008}$ (76.7) & $0.698{\scriptstyle\pm.017}$ (78.7) & $0.270{\scriptstyle\pm.006}$ (74.7) \\
EPIG, fixed top-$k$ & $0.216{\scriptstyle\pm.013}$ (51.0) & $0.167{\scriptstyle\pm.013}$ (28.3) & $0.069{\scriptstyle\pm.007}$ (48.3) \\
\bottomrule
\end{tabular}
\label{tab:infotheory}
\end{table}

\subsection{A direct test of the mode-aligned-target assumption}\label{app:modealign}
Proposition~\ref{prop:supervised} rests on one assumption --- the target varies only
across modes --- and makes two claims: under squared loss, the excess risk of the best
in-span predictor equals the uncovered mode mass; under an argmax readout, an uncovered
mode is predicted at chance. We test the assumption and both claims directly on the five
modes-$=$-labels datasets (three seeds; Table~\ref{tab:modealign}).

\begin{table}[!b]
\centering
\caption{\textbf{A direct test of the mode-aligned-target assumption}
(Proposition~\ref{prop:supervised}). $A$: share of target variance constant within
label-free modes ($k$-means, $r=C$) or classes; slope: in-span excess risk per unit of
uncovered class mass at $k=C$ (hard-mode value $1$).}
\footnotesize\setlength{\tabcolsep}{4pt}
\begin{tabular}{l cc cc c}
\toprule
dataset & $A$ (geometry modes) & $A$ (classes) & slope vs.\ $U(S)$ & $R^2$ & acc.\ on uncovered \\
\midrule
CIFAR-10 & $0.737$ & $0.795$ & $0.278{\scriptstyle\pm.012}$ & $0.84$ & $0.046$ \\
Oxford Pets & $0.707$ & $0.732$ & $0.215{\scriptstyle\pm.005}$ & $0.95$ & $0.000$ \\
CIFAR-100 & $0.494$ & $0.503$ & $0.047{\scriptstyle\pm.003}$ & $0.75$ & $0.000$ \\
Flowers-102 & $0.720$ & $0.740$ & $0.164{\scriptstyle\pm.002}$ & $0.98$ & $0.000$ \\
Food-101 & $0.456$ & $0.438$ & $0.028{\scriptstyle\pm.002}$ & $0.73$ & $0.000$ \\
\bottomrule
\end{tabular}
\label{tab:modealign}
\end{table}

\emph{The assumption.} Let $f^\star$ be the learner's full-pool predictor (the ridge probe
fit on every pool point, the $f^\star=\theta_\star^\top\varphi$ of the proposition) and let
the modes be found \emph{label-free} by $k$-means on the features ($r=C$, as in the
diagnostic). The mode-aligned fraction
$A=1-\E\|f^\star(x)-\bar f^\star_{\mathrm{mode}(x)}\|^2/\E\|f^\star(x)-\E f^\star\|^2$,
with $\bar f^\star_m$ the mean of $f^\star$ over mode $m$, equals $1-O(\eta^2)$ in the
proposition's spiked model. On real features it is $0.46$--$0.74$: the targets are
substantially but not perfectly mode-aligned, most on CIFAR-10, Oxford Pets, and
Flowers-102, least on CIFAR-100 and Food-101. Alignment tracks the coverage gains: across
the three label-free datasets, facility location's gain over random
(Table~\ref{tab:unlab}: Flowers-102 $+0.20$, CIFAR-100 $+0.13$, Food-101 $+0.08$) orders
them as $A$ does ($0.72$, $0.49$, $0.46$). With coarse CIFAR-100 labels, the coarse target
aligns more with fine-level geometry modes ($r=100$: $A=0.63$) than with a coarse
clustering ($r=20$: $0.45$) --- mode-aligned at the fine level while the labels are
coarse, exactly where the boundary prescribes coverage within coarse strata
(\S\ref{sec:boundary}).

\emph{The squared-loss claim.} Holding the budget at $k=C$, we choose \emph{which} classes
are covered --- $25$, $50$, $75$, or $100\%$ of them, the $k$ points spread evenly over
the covered classes, $8$ draws per level and seed --- and regress the excess risk $L(S)$
of the best in-span kernel predictor (population-oracle coefficients on the RBF sections
of the selected points, one-hot targets) on the uncovered class mass $U(S)$. The
hard-mode prediction is $L=U$. The relation is linear ($R^2=0.73$--$0.98$) but
attenuated, with slope $0.03$--$0.28$ rather than $1$: in real RBF geometry an uncovered
class is partly predictable from covered neighbors, and the attenuation is strongest
where alignment is weakest (the slope and $A$ rank the datasets almost identically, Spearman $0.9$). The
squared-loss equality is thus the hard-mode limit of an attenuated linear relation, consistent with
the soft-mode analysis (App.~\ref{app:soft}).

\emph{The classification claim.} The ridge probe trained on the same subsets is correct on
only $0$--$5\%$ of the test points of uncovered classes, at or below chance: under an
argmax readout the mass of every uncovered class is lost, however soft the geometry. This
is the mechanism behind the coverage--accuracy relation of Figure~\ref{fig:covacc}.

\end{document}